\documentclass{article}

\usepackage{microtype}
\usepackage{graphicx}
\usepackage{subcaption}
\usepackage{booktabs} 
\usepackage{mathtools}
\usepackage{csquotes}
\usepackage{enumitem}
\usepackage{tcolorbox}
\usepackage{pifont}
\usepackage{tabularx}
\usepackage{multirow}
\usepackage{graphicx}
\usepackage{threeparttable}

\usepackage{hyperref}

\usepackage[preprint]{icml2026}

\usepackage{amsmath}
\usepackage{amssymb}
\usepackage{mathtools}
\usepackage{amsthm}
\usepackage{bm}
\usepackage{amsfonts}

\usepackage{tikz}
\usetikzlibrary{shapes, arrows.meta, positioning, calc, backgrounds, shadows}
\usepackage[capitalize,noabbrev]{cleveref}

\theoremstyle{plain}
\newtheorem{theorem}{Theorem}[section]

\newtheorem{lemma}[theorem]{Lemma}
\newtheorem{corollary}[theorem]{Corollary}
\theoremstyle{definition}
\newtheorem{definition}[theorem]{Definition}
\newtheorem{assumption}[theorem]{Assumption}
\theoremstyle{remark}

\usepackage[textsize=tiny]{todonotes}

\icmltitlerunning{Rethinking LoRA for Data Heterogeneous Federated Learning: Subspace and State Alignment}

\def\eqref#1{equation~\ref{#1}}

\def\1{\bm{1}}

\def\mA{{\bm{A}}}
\def\mB{{\bm{B}}}

\DeclareMathAlphabet{\mathsfit}{\encodingdefault}{\sfdefault}{m}{sl}
\SetMathAlphabet{\mathsfit}{bold}{\encodingdefault}{\sfdefault}{bx}{n}

\newcommand{\methodname}{\texttt{FedGaLore}}

\begin{document}

\twocolumn[
  \icmltitle{Rethinking LoRA for Data Heterogeneous Federated Learning: \\Subspace and State Alignment}



  \icmlsetsymbol{equal}{*}

  \begin{icmlauthorlist}
    \icmlauthor{Hongyi Peng}{NTU}
    \icmlauthor{Han Yu}{NTU}
    \icmlauthor{Xiaoxiao Li}{UBC,Vector}
    \icmlauthor{Qiang Yang}{PolyU}
  \end{icmlauthorlist}

\icmlaffiliation{NTU}{College of Computing and Data Science, Nanyang Technological University, Singapore.}
\icmlaffiliation{UBC}{Department of Electrical and Computer Engineering, The University of British Columbia, Vancouver, BC, Canada.}
\icmlaffiliation{Vector}{Vector Institute, Canada}
\icmlaffiliation{PolyU}{The Hong Kong Polytechnic University}

\icmlcorrespondingauthor{Han Yu}{ han.yu@ntu.edu.sg}

  \icmlkeywords{Federated Learning,  Data Heterogeneity, Para}

  \vskip 0.3in
]



\printAffiliationsAndNotice{}  

\begin{abstract}
Low-Rank Adaptation (LoRA) is widely used for federated fine-tuning. Yet under non-IID settings, it can substantially underperform full-parameter fine-tuning. Through with-high-probability robustness analysis, we uncover that this gap can be attributed to two coupled mismatches: (i) \emph{update-space mismatch}, where clients optimize in a low-rank subspace but aggregation occurs in the full space; and (ii) \emph{optimizer-state mismatch}, where unsynchronized adaptive states amplify drift across rounds. We propose \methodname{}, which combines client-side GaLore-style gradient-subspace optimization with server-side drift-robust synchronization of projected second-moment states via spectral shared-signal extraction, to address this challenge. Across NLU, vision, and NLG benchmarks, \methodname{} improves robustness and accuracy over state-of-the-art federated LoRA baselines in non-IID settings.
\end{abstract}

\section{Introduction}\label{sec:intro}
Foundation models (FMs) have reshaped machine learning via the pretrain--finetune paradigm \citep{bommasani2021opportunities,brown2020language,llama,kirillov2023segment}. However, their scale \citep{kaplan2020scaling} makes full fine-tuning (FFT) prohibitively expensive. This has motivated parameter-efficient fine-tuning (PEFT), where Low-Rank Adaptation (LoRA) \citep{hu2021lora} has become the de facto standard. LoRA freezes pretrained weights and learns low-rank adapter factors $\mA$ and $\mB$, thereby greatly reducing trainable parameters (and optimizer-state memory) while retaining strong centralized performance. These benefits have spurred federated LoRA variants. However, federated learning (FL) is sensitive to statistical heterogeneity: non-IID clients might produce biased and divergent local updates. As a result, while LoRA can match the performance of FFT in centralized training \citep{hu2021lora}, it often degrades substantially under non-IID federated settings \citep{babakniya2023slora, wu2025survey_peft_fl}.

\begin{figure}[t]
    \centering
    \begin{subfigure}[b]{0.49\linewidth}
        \centering
        \includegraphics[width=\textwidth]{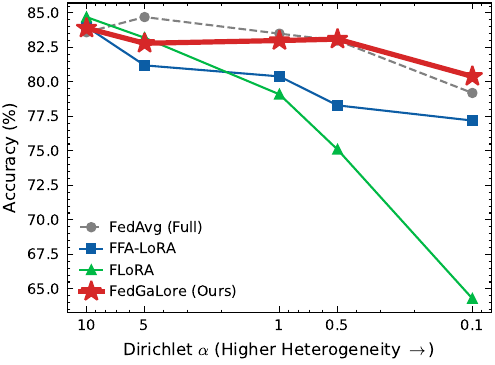}
    \end{subfigure}
    \hfill
    \begin{subfigure}[b]{0.49\linewidth}
        \centering
        \includegraphics[width=\textwidth]{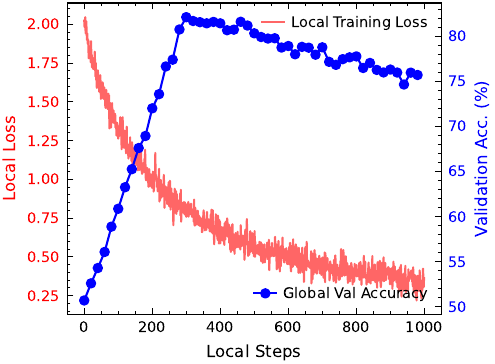}
    \end{subfigure}
\caption{\textbf{Left:} Under increasing data heterogeneity (smaller Dirichlet $\alpha$), representative federated LoRA baselines degrade sharply, while \methodname{} remains stable and approaches full fine-tuning performance. 
\textbf{Right:} With local adaptive optimizers, client training loss can decrease while global validation performance stagnates or degrades, indicating optimizer-state mismatch.}
 \label{fig:fedlora_non_iid}
 \vspace{-8pt}
\end{figure}

Figure~\ref{fig:fedlora_non_iid} shows that federated LoRA fine-tuning underperforms full fine-tuning under non-IID FL settings. The gap widens as data heterogeneity increases. Prior work mitigates this by further restricting the update space (e.g., FFA-LoRA shares a factor across clients), which can improve robustness under severe data heterogeneity, with reduced IID performance \citep{sun2024improving}. We also observe a local--global mismatch with adaptive optimizers (e.g., Adam/AdamW): local training loss decreases while global validation performance improves little or even degrades. These observations point to two coupled mechanisms: (i) misalignment of client updates in the parameter space, and (ii) drift of adaptive optimizer states across clients and rounds.

To formalize this and complement standard expectation-based analysis, we develop a with-high-probability (W.H.P.) robustness framework that controls the probability of rare but catastrophic deviations under heterogeneity. We view federated fine-tuning as a perturbed implementation of a centralized optimizer, and decompose each round into client training ($\mathcal{T}$), server aggregation ($\mathcal{A}$) and state synchronization ($\mathcal{S}$). This lens isolates two failure modes:
\textbf{(1) Update-space mismatch:} LoRA constrains each client update to a low-rank subspace, while $\mathcal{A}$ aggregates in the full parameter space.
\textbf{(2) Optimizer-state mismatch:} inconsistencies in adaptive states (arising from imperfect $\mathcal{S}$ and amplified by $\mathcal{T}$) propagate through local dynamics and amplify client drift.
Together, these effects explain the brittleness of federated LoRA under data heterogeneity.

To address these limitations, we propose \methodname{}. It couples \emph{gradient-subspace client optimization} with \emph{drift-robust state synchronization}. On the client side, we replace LoRA’s fixed parameter-subspace updates with GaLore-style gradient-subspace optimization: each client computes gradients on the adapted modules and updates in an adaptively estimated rank-$r$ subspace, preserving LoRA-like parameter efficiency while improving aggregation robustness. Clients use a GaLore-based AdamW variant to maintain local adaptivity.
On the server side, we mitigate optimizer-state mismatch with minimal communication by synchronizing only the \emph{projected second-moment state} that defines the preconditioner. Specifically, clients upload low-rank projected second moments, and the server applies an AJIVE filter \citep{feng2018ajive} to extract the shared component across clients and broadcasts it as the next-round initialization, stabilizing training under non-IID data.
Extensive experimental evaluation has been conducted on the NLU, vision and NLG benchmarks. The results show that \methodname{} significantly improves robustness to data heterogeneity and stability over state-of-the-art federated LoRA baselines.


\section{Related Works}
LoRA \citep{hu2021lora} has inspired many extensions, including rank-adaptive/structured adapters \citep{zhang2023adalora,kopiczko2023vera,liu2024dora} and improved initialization \citep{lialin2023relora,meng2024pissa}. More recently, \textbf{GaLore} \citep{zhao2024galore} shifts from parameter-space adapters to \emph{gradient-subspace optimization}: it projects gradients into an evolving low-rank subspace, preserving access to full-gradient information while keeping updates low-rank. This has motivated follow-up work on theory and variants \citep{he2024subspace_improve_galore_guarantee,pan2025unbiased_galore,su2025galore2}, and comparative analyses suggesting more favorable optimization behavior than fixed parameter constraints in LoRA \citep{hao2024flora,liu2025optimization}.

FL must cope with client drift induced by statistical heterogeneity. FedAvg \citep{mcmahan2017communication} and drift-control methods such as SCAFFOLD \citep{karimireddy2020scaffold} and FedDyn \citep{AcarZNMWS21feddyn} primarily analyze local SGD, while FedOpt introduces server-side adaptivity \citep{reddi2020adaptive_fedopt}. More recent work studies \emph{local} adaptive optimizers \citep{li2023fedda,tang2024fedlion}, but synchronizing their internal states remains nontrivial \citep{douillard2023diloco}. On the theory side, convergence analyses increasingly cover broader heterogeneity \citep{cheng2023momentum_convergence}, yet most results are in expectation and can obscure rare but destabilizing events. We adopt a with-high-probability (W.H.P.) lens \citep{cutkosky2021highprob_first} to explicitly control such outliers, which is particularly relevant for robust federated fine-tuning.

Federated LoRA is \emph{not} a drop-in replacement for FFT-based FL: heterogeneity remedies developed for full fine-tuning do not necessarily transfer because LoRA changes the update space and effective local capacity. Empirically, combining LoRA with standard drift-control can be inconsistent \citep{ye2024openfedllm}. Moreover, many federated LoRA methods adopt local adaptive optimizers (Adam/AdamW) for faster client progress, but rarely account for optimizer-state drift across clients and rounds. Table~\ref{tab:existing_method} highlights that prior work varies in optimizer choice and whether aggregation is performed on low-rank factors or lifted full-space updates, yet few explicitly synchronize adaptive states.

\begin{table}[!h]
\centering
\caption{\textbf{Representative federated LoRA methods.} Local optimizer, aggregation, and whether optimizer states are synchronized.}
\label{tab:existing_method}
\footnotesize
\setlength{\tabcolsep}{4pt}
\renewcommand{\arraystretch}{1.08}
\begin{minipage}{\columnwidth}
\centering
\begin{tabular*}{\columnwidth}{@{\extracolsep{\fill}}l l l c@{}}
\toprule
\textbf{Method} & \textbf{Opt.} & \textbf{Agg.} & \textbf{Sync} \\
\midrule
FedIT \citep{zhang2024towards}       & Adam  & factors      & No \\
FFA-LoRA \citep{sun2024improving}   & SGD   & factors      & No \\
LoRA-Fair \citep{bian2024lora-fair} & SGD   & factors      & No \\
FLoRA \citep{wang2024flora}         & AdamW & lift $\Delta W$ & No \\
FR-LoRA \citep{yanfrlora}           & AdamW & lift $\Delta W$ & No \\
\bottomrule
\end{tabular*}
\end{minipage}
\end{table}

\section{Preliminaries}
\paragraph{Notation.}
\DeclarePairedDelimiter{\norm}{\lVert}{\rVert}
\newcommand{\mat}[1]{\bm{#1}}
\NewDocumentCommand{\client}{m O{i} O{k} O{t}}{{#1}_{#4}^{#2,#3}}
\newcommand{\grad}{\mat{g}}
\newcommand{\drift}{\mat{c}}
\newcommand{\param}{\mat{\theta}}
\newcommand{\noise}{\mat{\xi}}
\newcommand{\data}{\mathcal{D}}
\newcommand{\loss}{\mathcal{L}}
\newcommand{\prob}{\mathbb{P}}
\DeclarePairedDelimiter{\expect}{\mathbb{E}[}{]}

\newcommand{\state}{\mat{S}}

\newcommand{\Cdrift}{\client{\drift}}
\newcommand{\Cgrad}{\client{\grad}}
\newcommand{\Cparam}{\client{\param}}
\newcommand{\Cnoise}{\client{\noise}}
\newcommand{\Cdata}{\data_{i}}
\newcommand{\Closs}{\loss_{i}}
\newcommand{\Cstate}{\client{\state}}
For $a\in\mathbb{R}^d$, $\|a\|$ denotes the Euclidean norm and $\|a\|_{\infty}\coloneqq \max_i |a_i|$ the infinity norm. For vectors, $\odot$ and $/$ denote element-wise multiplication and division; $a^2$, $\sqrt{a}$, and $|a|$ are applied coordinate-wise. Matrices are denoted by $\mat{A}\in\mathbb{R}^{d_1\times d_2}$; $\|\mat{A}\|$ is the spectral norm and $\|\mat{A}\|_F$ the Frobenius norm. 

\paragraph{Problem Setup}
We consider an FL system with $M$ clients, where client $i$ has a local dataset $\mathcal{D}_i=\{(x_n,y_n)\}_{n=1}^{N_i}$ and objective $F_i(\theta)$ (population or empirical risk). Client data may be non-IID: each client draws samples from a potentially different distribution $\mathbb{P}_i$, inducing biased and divergent local updates under heterogeneous training.
Let $\theta\in\Theta\subseteq\mathbb{R}^d$ denote model parameters and define the global objective
$f(\theta)\coloneqq \sum_{i=1}^M p_i F_i(\theta)$,
where $p_i\ge 0$, $\sum_{i=1}^M p_i=1$, and $p_i$ is typically proportional to data size or client sampling probability.\footnote{We assume full participation unless stated otherwise.} 
Let $g^{i,k}_t$ denote the stochastic gradient on client $i$ at round $k$ and local step $t$ (an unbiased estimator of $\nabla F_i(\theta^{i,k}_t)$). We decompose:
\begin{equation}\label{eq:decomposition}
g^{i,k}_t \;=\; \nabla f(\theta^{i,k}_t) \;+\; c^{i,k}_t \;+\; \xi^{i,k}_t.
\end{equation}
$c^{i,k}_t \coloneqq \nabla F_i(\theta^{i,k}_t)-\nabla f(\theta^{i,k}_t)$ captures client drift. $\xi^{i,k}_t \coloneqq g^{i,k}_t-\nabla F_i(\theta^{i,k}_t)$ is mini-batch noise. For any fixed $\theta$, $\sum_{i=1}^M p_i\big(\nabla F_i(\theta)-\nabla f(\theta)\big)=0$ by definition. Under local training $c^{i,k}_t$ is evaluated along different client trajectories.

\begin{definition}[Local Training Operator]
The local training operator $\mathcal{T}_i(\bar{\theta}_k, S^{i,k}_{0}; T)$ runs $T$ local steps on client $i$ starting from global model $\bar{\theta}_k$ and initial optimizer state $S^{i,k}_{0}$ (e.g., SGD, Adam \citep{kingma2014adam}, AdamW \citep{loshchilov2017adamw}). It outputs
\begin{equation}
(\theta^{i,k}_{T}, S^{i,k}_{T}) \;=\; \mathcal{T}_i(\bar{\theta}_k, S^{i,k}_{0}; T).
\end{equation}
\end{definition}

\begin{definition}[Server Aggregation Operator]
Given participating clients $\mathcal{P}_k\subseteq[M]$ and their terminal models $\{\theta^{i,k}_{T}\}_{i\in\mathcal{P}_k}$, the aggregation operator $\mathcal{A}$ produces
\begin{equation}
\bar{\theta}_{k+1} \;=\; \mathcal{A}\big(\{\theta^{i,k}_{T}\}_{i\in\mathcal{P}_k}\big),
\end{equation}
with a canonical choice being FedAvg:
$\bar{\theta}_{k+1}=\sum_{i\in\mathcal{P}_k}\tilde p_i\,\theta^{i,k}_{T}$, where
$\tilde p_i \coloneqq \frac{p_i}{\sum_{j\in\mathcal{P}_k}p_j}$.
\end{definition}

\begin{definition}[State Synchronization Protocol]
The synchronization operator $\mathcal{S}$ specifies how optimizer states are shared or aggregated across rounds. Given terminal client states $\{S^{i,k}_{T}\}_{i\in\mathcal{P}_k}$ (and optionally server state $\bar S_k$):
\begin{align}
\bar S_{k+1} \;&=\; \mathcal{S}\big(\{S^{i,k}_{T}\}_{i\in\mathcal{P}_k}, \bar S_k\big),\\
S^{i,k+1}_0 &\leftarrow \mathrm{InitState}(\bar S_{k+1}).
\end{align}
Common choices: 1) \emph{none} (clients reinitialize each round), and 2) \emph{server-only} (the server maintains adaptive states).
\end{definition}

This modular view covers many FL algorithms. FedAvg \citep{mcmahan2017communication} uses SGD for $\mathcal{T}$, FedAvg for $\mathcal{A}$, and \emph{none} for $\mathcal{S}$. SCAFFOLD \citep{karimireddy2020scaffold} introduces control variates in $\mathcal{T}$ that must be synchronized by $\mathcal{S}$. FedOpt/FedAdam \citep{reddi2020adaptive_fedopt} typically keeps $\mathcal{T}$ stateless while implementing adaptivity at the server through $\mathcal{A}$ together with server-maintained states in $\mathcal{S}$.

\paragraph{Assumptions}\label{sec:assumptions}
We analyze federated fine-tuning over a region $Q \subseteq \Theta$ that contains the iterates (i.e., $\theta^{i,k}_t \in Q$ for all $i,k,t$). Unless stated otherwise, all assumptions below hold on $Q$. This restricted-region viewpoint is standard for modern over-parameterized networks: although the global landscape is non-convex, the optimizer often remains in a well-behaved region where a local PL condition holds (often termed PL$^*$) \citep{liu2022overparameterized_pl}.

\begin{assumption}[Lower-boundedness]\label{ass:lowerbounded}
The objective is lower bounded on $Q$:
$f^* \coloneqq \inf_{\theta \in Q} f(\theta) > -\infty$.
\end{assumption}

\begin{assumption}[Smoothness]\label{ass:smoothness}
Each client objective $F_i$ is $L$-smooth on $Q$:
\begin{equation}
\|\nabla F_i(x)-\nabla F_i(y)\| \le L\|x-y\|,\quad \forall x,y\in Q,\ \forall i\in[M].
\end{equation}
In particular, $f=\sum_i p_i F_i$ is also $L$-smooth on $Q$.
\end{assumption}

\begin{assumption}[Polyak--\L{}ojasiewicz (PL)]\label{ass:pl-condition}
There exists $\mu>0$ such that for all $\theta\in Q$,
\begin{equation}
\|\nabla f(\theta)\|^2 \ge 2\mu\big(f(\theta)-f^*\big).
\end{equation}
\end{assumption}

\begin{assumption}[Bounded heterogeneity]\label{ass:bounded_heterogeneity}
There exist constants $B\ge 1$ and $H\ge 0$ such that for all $\theta\in Q$,
\begin{equation}
\sum_{i=1}^M p_i \|\nabla F_i(\theta)\|^2 \le H^2 + B^2\|\nabla f(\theta)\|^2.
\end{equation}
\end{assumption}

\begin{assumption}[Bounded local gradients]\label{ass:bounded_grad}
There exists $G\ge 0$ such that for all $i\in[M]$ and $\theta\in Q$, $\|\nabla F_i(\theta)\| \le G$.
\end{assumption}

Assumption~\ref{ass:bounded_heterogeneity} is standard in FL for controlling non-IID drift \citep{karimireddy2020scaffold, wang2021field_guide}. 
Assumption~\ref{ass:bounded_grad} is a technical condition for bounded increments in high-probability analyses and is satisfied when using norm-wise gradient clipping with threshold $G$ \citep{goodfellow2016deep, reddi2020adaptive_fedopt, wu2023FAFED, wang2022communication}.

\begin{assumption}[Sub-Gaussian mini-batch noise]\label{ass:subgaussian}
For all $(i,k,t)$, the noise $\xi_t^{i,k}$ is conditionally mean-zero and $\sigma$-sub-Gaussian given $\mathcal{F}_{k,t-1}$ (Definition~\ref{def:subg}).
\end{assumption}
This assumption is standard for W.H.P.\ bounds via martingale concentration \citep{scaman2020robustness_sub_gaussian}. It is also common in stochastic-dynamics views of optimization \citep{welling2011bayesian_langevin_dynamics, raginsky2017nonconvex_langevin_dynamics}. Our results can be extended to weaker moment assumptions following, e.g., \citet{cheng2024convergence_whp}.

\section{A W.H.P. Lens for Federated Learning}
\paragraph{Why W.H.P.\ robustness?}
Expectation-based analyses describe \emph{average} behavior but can obscure instability under severe heterogeneity. In federated fine-tuning, client drift and mini-batch noise can interact to produce rare but catastrophic trajectories that leave a stable basin and derail aggregation. With-high-probability (W.H.P.) guarantees address this by controlling the entire optimization path with confidence $1-\delta$ over the stochasticity. For instance, if a method converges for $90\%$ of realizations but diverges for the remaining $10\%$, an in-expectation guarantee may still look favorable, whereas a W.H.P.\ statement cannot hold at confidence levels above $0.9$.

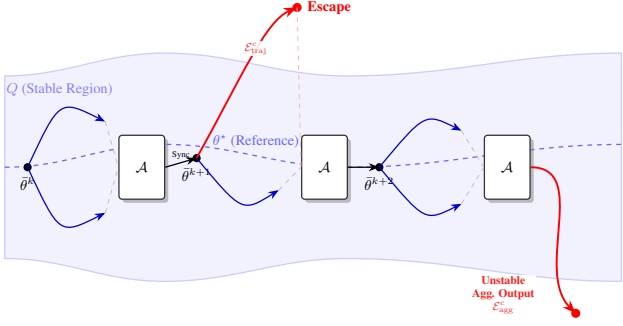
\begin{figure}[t]
\centering
\resizebox{\linewidth}{!}{%
\begin{tikzpicture}[
    >=Stealth,
    font=\small,
    region_q/.style={fill=blue!5, draw=blue!20, thick},
    ref_traj/.style={blue!60, dashed, thick},
    local_traj/.style={blue!80!black, thick, ->, smooth, tension=0.6},
    escape_traj/.style={red, very thick, ->, smooth, tension=0.6},
    agg_box/.style={
        draw=black!80, 
        thick, 
        fill=white, 
        rectangle, 
        minimum width=1.0cm, 
        minimum height=1.4cm, 
        rounded corners=2pt,
        drop shadow
    },
    init_point/.style={circle, fill=black, inner sep=1.8pt},
    bad_point/.style={circle, fill=red, inner sep=2pt},
    label_box/.style={fill=white, inner sep=1pt, opacity=0.9, text=black, scale=0.8, align=center}
]

    \filldraw[region_q] 
        (0, 2.0) to[out=0, in=180] (3.5, 2.5) to[out=0, in=180] (7.5, 2.0) to[out=0, in=180] (13.5, 2.5) -- 
        (13.5, -2.5) to[out=180, in=0] (7.5, -2.0) to[out=180, in=0] (3.5, -2.5) to[out=180, in=0] (0, -2.0) -- cycle;
        
    \node[blue!50] at (1.2, 1.7) {$Q$ (Stable Region)};

    \draw[ref_traj] (0, 0) to[out=0, in=180] (3.5, 0.5) to[out=0, in=180] (7.5, 0) to[out=0, in=180] (13.5, 0.5);
    \node[blue!60, above] at (5.5, 0.3) {$\theta^\star$ (Reference)};

    \coordinate (start_k) at (0.5, 0);
    \node[init_point, label=below:{$\bar{\theta}^k$}] at (start_k) {};

    \draw[local_traj] (start_k) .. controls (1.2, 1.5) .. (2.2, 1.0) coordinate (end_k_1);
    \draw[local_traj] (start_k) .. controls (1.2, -1.5) .. (2.2, -1.0) coordinate (end_k_2);
    
    \node[agg_box] (box_k) at (3.0, 0) {$\mathcal{A}$};
    \draw[gray!50, dashed] (end_k_1) -- (box_k.west);
    \draw[gray!50, dashed] (end_k_2) -- (box_k.west);

    \coordinate (start_k1) at (4.2, 0.2);
    \draw[->, thick] (box_k.east) -- (start_k1) node[midway, above, scale=0.6] {Sync};
    \node[init_point, label=below:{$\bar{\theta}^{k+1}$}] at (start_k1) {};

    \draw[local_traj] (start_k1) .. controls (5.2, -1.0) .. (6.0, -0.5) coordinate (end_k1_good);

    \draw[escape_traj] (start_k1) .. controls (5.2, 2.0) and (5.7, 3.2) .. (6.4, 3.5) coordinate (end_k1_bad);
    \node[bad_point, label=right:{\color{red}\textbf{Escape}}] at (end_k1_bad) {};
    \node[label_box, text=red] at (5.5, 2.6) {$\mathcal{E}_{\mathrm{traj}}^c$};

    \node[agg_box] (box_k1) at (7.0, 0) {$\mathcal{A}$};
    \draw[gray!50, dashed] (end_k1_good) -- (box_k1.west);
    \draw[red!30, dashed] (end_k1_bad) -- (box_k1.west);

    \coordinate (start_k2) at (8.2, 0);
    \draw[->, thick] (box_k1.east) -- (start_k2);
    \node[init_point, label=below:{$\bar{\theta}^{k+2}$}] at (start_k2) {};

    \draw[local_traj] (start_k2) .. controls (9.2, 1.2) .. (10.0, 0.8) coordinate (end_k2_1);
    \draw[local_traj] (start_k2) .. controls (9.2, -1.2) .. (10.0, -0.8) coordinate (end_k2_2);

    \node[agg_box] (box_k2) at (11.0, 0) {$\mathcal{A}$};
    \draw[gray!50, dashed] (end_k2_1) -- (box_k2.west);
    \draw[gray!50, dashed] (end_k2_2) -- (box_k2.west);

    \coordinate (bad_out) at (12.5, -3.2); 
    \draw[->, red, very thick] (box_k2.east) to[out=0, in=135] (bad_out);
    \node[bad_point] at (bad_out) {};
    
    \node[label_box, text=red, anchor=west] at (10.2, -2.8) {\textbf{Unstable}\\\textbf{Agg. Output}\\$\mathcal{E}_{\mathrm{agg}}^c$};

\end{tikzpicture}%
} 

\caption{\textbf{Trajectory robustness and failure modes.}
The stable region $Q$ (blue tube) surrounds a centralized reference solution $\theta^\star$ (dashed).
\textbf{Left (round $k$):} Stable operation; all client trajectories remain in $Q$.
\textbf{Middle (round $k{+}1$):} \emph{Local escape} ($\mathcal{E}_{\mathrm{traj}}^{c}$, red), where client drift causes a trajectory to exit $Q$.
\textbf{Right (round $k{+}2$):} \emph{Aggregation instability} ($\mathcal{E}_{\mathrm{agg}}^{c}$), where the aggregated model leaves $Q$ even though all client endpoints lie in $Q$.}
\label{fig:whp_horizontal_fitted}
\end{figure}

Inspired by recent W.H.P.\ analyses of local adaptive optimization \citep{cheng2024convergence_whp}, we study FL through a trajectory-robustness lens. We compare federated iterates $\{\bar\theta^k\}_{k\ge 0}$ to a centralized reference trajectory $\{\theta_t^\star\}$ within a \emph{stable region} $Q$ where $f$ is well behaved (Assumptions~\ref{ass:smoothness}--\ref{ass:pl-condition}). For round $k$, let $\Theta^{i,k}$ denote client $i$'s local trajectory. As illustrated in Fig.~\ref{fig:whp_horizontal_fitted}, we decompose stability into two coupled success events:
\begin{itemize}[noitemsep, topsep=0pt, leftmargin=*]
\item \textbf{Local containment $\mathcal{E}_{\mathrm{traj}}(K)$:} all client trajectories stay in $Q$ for rounds $k\le K$.
\item \textbf{Aggregation stability $\mathcal{E}_{\mathrm{agg}}(K)$:} the aggregated iterate $\bar\theta^{k+1}$ stays in $Q$ for rounds $k\le K$.
\end{itemize}
Controlling the failure probabilities of $\mathcal{E}_{\mathrm{traj}}(K)$ and $\mathcal{E}_{\mathrm{agg}}(K)$ implies, via a union bound and induction, that the global trajectory remains in $Q$ with probability at least $1-\delta$, enabling standard contraction arguments throughout training. This modular roadmap isolates failure modes: $\mathcal{E}_{\mathrm{traj}}$ reflects drift and optimizer-state mismatch, while $\mathcal{E}_{\mathrm{agg}}$ captures aggregation instability.

\subsection{Update-space Mismatch}\label{sec:update_space_mismatch}
We analyze aggregation stability, which governs $\mathcal{E}_{\mathrm{agg}}(K)$: even if all clients remain in a stable region, aggregation might push the global iterate outside it. Given round $k$ and assume $\theta^{i,k}_T\in Q$ for all participating clients, the question is whether the server update $\bar\theta^{k+1}=\mathcal{A}(\{\theta^{i,k}_T\})$ still lies in $Q$. Under convexity this is immediate: if $Q$ is convex and $\mathcal{A}$ is a convex combination, aggregation cannot leave $Q$.
\begin{lemma}[Aggregation preserves convex regions]\label{lem:convex_Q}
Let $Q$ be convex. If $\theta_i\in Q$ for all $i\in[M]$ and $\bar\theta=\sum_{i=1}^M \tilde p_i \theta_i$ with $\tilde p_i\ge 0$ and $\sum_i \tilde p_i=1$, then $\bar\theta\in Q$.
\end{lemma}

Lemma~\ref{lem:convex_Q} explains why aggregation is rarely a bottleneck in convex analyses, but this closure property typically fails for deep networks where stable regions (e.g., PL basins) are local and non-convex \citep{liu2022overparameterized_pl}. Empirically, mode-connectivity results show that independently trained solutions can often be connected by low-loss paths, sometimes making linear interpolation near-optimal \citep{garipov2018loss_mode_connectivity,draxler2018essentially_mode_connectivity_2}. This motivates weight averaging and related model-merging methods \citep{ainsworth2022git_rebasin,yang2024model_merge}, and is supported by recent theory under additional conditions \citep{ferbach2024proving_linear_mode_connectivity,kuditipudi2019explaining_linear_mode_connectivty}. Together, these observations help explain why full-parameter FedAvg is often stable: weighted averages of nearby solutions frequently remain within (or close to) a well-behaved region.

\paragraph{Federated LoRA differs from FFT.}
For a linear module, LoRA parameterizes $\mat{W}=\mat{W}^{(0)}+\mat{B}\mat{A}$ with rank $r\ll \min\{d_{\mathrm{out}},d_{\mathrm{in}}\}$, so each client update $\Delta\mat{W}=\mat{B}\mat{A}$ lies in a nonconvex rank-$\le r$ set. Early federated LoRA methods aggregate \emph{factors}: FedIT \citep{zhang2024towards} uses
$\Delta\bar{\mat{W}}^{k}=(\sum_i \tilde p_i\,\mat{B}^{i,k}_{T})(\sum_i \tilde p_i\,\mat{A}^{i,k}_{T})$,
and FFA-LoRA \citep{sun2024improving} fixes one factor so
$\Delta\bar{\mat{W}}^{k}=(\sum_i \tilde p_i\,\mat{B}^{i,k}_{T})\mat{A}^{(0)}$.
These enforce a rank-$\le r$ aggregate (\emph{low-rank-space aggregation}). In contrast, \emph{full-space aggregation} lifts each client adapter to a weight delta, and averages in the ambient space \citep{yanfrlora,wang2024flora}:
\[
\Delta \bar{\mat{W}}^{k}=\sum_{i\in\mathcal{P}_k} \tilde p_i\,(\mat{B}^{i,k}_{T}\mat{A}^{i,k}_{T}),
\]
so $\mathrm{rank}(\Delta \bar{\mat{W}}^{k})$ can grow up to $|\mathcal{P}_k|r$ even when each client update is rank-$\le r$, injecting out-of-subspace components that can destabilize aggregation.

To formalize the geometric tension, consider the rank-$r$ manifold (away from rank-deficient points)
\[
\mathcal{M}_r \coloneqq \{\Delta \mat{W}\in\mathbb{R}^{d_{\mathrm{out}}\times d_{\mathrm{in}}}:\ \mathrm{rank}(\Delta \mat{W})=r\},
\]
with ambient dimension $D_{\mathrm{amb}}=d_{\mathrm{out}}d_{\mathrm{in}}$, manifold dimension $d_{\mathcal{M}}=r(d_{\mathrm{out}}+d_{\mathrm{in}}-r)$, and codimension $\mathrm{codim}(\mathcal{M}_r)=(d_{\mathrm{out}}-r)(d_{\mathrm{in}}-r)$.
To make volumes finite, we restrict to a bounded Frobenius ball $\mathbb{B}_F(\rho)$ and define the tube
\[
\mathcal{M}_r^{(R)} \coloneqq \{\Delta \mat{W}\in \mathbb{B}_F(\rho):\ \mathrm{dist}_F(\Delta \mat{W},\mathcal{M}_r)\le R\}.
\]

\begin{theorem}[Weyl's tube formula (informal) \citep{gray2003tubes}]\label{thm:weyl_tube}
Let $\mathcal{M}$ be a compact $d_{\mathcal{M}}$-dimensional $C^2$ submanifold of $\mathbb{R}^{D_{\mathrm{amb}}}$. For sufficiently small $R$,
\begin{equation}
\begin{aligned}
\mathrm{Vol}(\mathcal{M}^{(R)})
= & \mathrm{Vol}(\mathcal{M})\,
   \mathrm{Vol}\!\big(B^{D_{\mathrm{amb}}-d_{\mathcal{M}}}(R)\big)\\
 + &O\!\big(R^{D_{\mathrm{amb}}-d_{\mathcal{M}}+1}\big),
\end{aligned}
\end{equation}
so $\mathrm{Vol}(\mathcal{M}^{(R)})$ scales as $R^{D_{\mathrm{amb}}-d_{\mathcal{M}}}$ to leading order.
\end{theorem}

Theorem~\ref{thm:weyl_tube} highlights an unforgiving geometry: when $\mathrm{codim}(\mathcal{M}_r)$ is large, an $R$-tube around rank-$r$ updates occupies a vanishingly small fraction of the ambient space, so small off-manifold components can push $\Delta\bar{\mat{W}}^k$ outside a ``stable tube'' $Q$.
A complementary view follows from Eckart--Young. Let $\mathcal{M}_{\le r}\coloneqq\{\Delta\mat{W}:\mathrm{rank}(\Delta\mat{W})\le r\}$:
\begin{equation}
\mathrm{dist}_F(\Delta\bar{\mat{W}}^k,\mathcal{M}_{\le r})
=\Big(\sum_{j>r}\sigma_j(\Delta\bar{\mat{W}}^k)^2\Big)^{1/2},
\end{equation}
where $\{\sigma_j(\cdot)\}$ are singular values in nonincreasing order; thus misaligned client updates induce nontrivial tail singular values beyond rank $r$.
We refer to this failure mode as \textbf{update-space mismatch}, which helps explain why more restrictive schemes (e.g., fixing one adaptation factor in FFA-LoRA \citep{sun2024improving}) can trade expressivity for stability as non-IID severity increases.

\subsection{Optimizer-state Mismatch}\label{sec:state_mismatch}
While update-space mismatch governs aggregation stability (event $\mathcal{E}_{\mathrm{agg}}$), local containment (event $\mathcal{E}_{\mathrm{traj}}$) is governed by the stochastic dynamics of client-side optimization.
Adaptive optimizers (e.g., Adam/AdamW \citep{kingma2014adam, hu2021lora}) are the de facto standard for fine-tuning large models; LoRA typically uses AdamW. In FL, clients optimize different non-IID objectives, so their adaptive states (e.g., first/second moments) evolve inconsistently across clients and rounds. We term this phenomenon \textbf{optimizer-state mismatch}. Such mismatch can bias local update directions and amplify drift/noise, increasing the probability that a client trajectory exits the stable region (event $\mathcal{E}_{\mathrm{traj}}^{c}$). Since most federated LoRA methods do not explicitly synchronize adaptive states, characterizing state mismatch propagation through local dynamics is critical.

We compare each client to a centralized reference optimizer, which induces a canonical state $(m^{\star,k}_0,v^{\star,k}_0)$ at the start of round $k$. Practical protocols approximate these states (e.g., pseudo-gradients from deltas \citep{reddi2020adaptive_fedopt, douillard2023diloco} or state aggregation \citep{sun2023efficient_adaptive_fl_second, liu2025fedadamw}); we model the resulting discrepancy as a bounded initialization error.

\begin{definition}[Biased state initialization]\label{def:state_bias_compact}
At the start of round $k$, client $i$ initializes $(m^{i,k}_0,v^{i,k}_0)$ while the reference uses $(m^{\star,k}_0,v^{\star,k}_0)$.
Assume $\|m^{i,k}_0-m^{\star,k}_0\|\le B_m$ and $\|v^{i,k}_0-v^{\star,k}_0\|\le B_v$ for all $i,k$, where $\|\cdot\|$ denotes $\ell_2$ for vectors and Frobenius norm for matrices (set $B_v=0$ for optimizers without $v$).
\end{definition}

Let $\varepsilon_{\mathrm{noise}}(\delta)$ denote the W.H.P.\ bound from Lemma~\ref{lem:noise_env} such that, with probability at least $1-\delta$,
\[
\max_{k\le K-1,\; t\le T,\; i\in[M]}\|\xi_t^{i,k}\|_2 \le \varepsilon_{\mathrm{noise}}(\delta).
\]

\begin{theorem}[W.H.P.\ local-containment radius]\label{thm:radius_local}
Assume (i) Assumptions~\ref{ass:smoothness} and~\ref{ass:bounded_grad}, (ii) the noise envelope above, and (iii) Definition~\ref{def:state_bias_compact}. Let $\{\theta_{t}^{\star,k}\}_{t=0}^{T}$ be the within-round reference trajectory started from $\bar\theta^k$ under the corresponding centralized update rule. Suppose $\eta\le \frac{1}{2LT}$ and standard hyperparameters are used (Appendix~\ref{appendix:local_training_operator}).

Then for any $\delta\in(0,1)$, with probability at least $1-\delta$ over mini-batch sampling, simultaneously for all rounds $k\le K-1$, all clients $i\in[M]$, and all local steps $t\le T$,
\[
\|\theta^{i,k}_t-\theta^{\star,k}_t\|_2 \le \mathcal{R}_{\mathrm{loc}}(\delta),
\]
where, up to absolute constants,
\[
\mathcal{R}_{\mathrm{loc}}(\delta)\;\lesssim\;
\mathcal{R}_{\mathrm{drift}}(\delta)+\mathcal{R}_{\mathrm{state}},
\]
with
$
\mathcal{R}_{\mathrm{drift}}(\delta)=
\begin{cases}
\eta T\big(G+\varepsilon_{\mathrm{noise}}(\delta)\big), & \mathrm{SGD/Momentum},\\[2pt]
\frac{\eta T}{\sqrt{\epsilon}}\big(G+\varepsilon_{\mathrm{noise}}(\delta)\big), & \mathrm{AdamW},
\end{cases} 
$

and

$
\mathcal{R}_{\mathrm{state}}=
\begin{cases}
0, & \mathrm{SGD},\\[2pt]
\eta\,\dfrac{B_m}{1-\beta_1},  & \mathrm{Momentum},\\[6pt]
\eta\!\left(\dfrac{B_m}{(1-\beta_1)\sqrt{\epsilon}}+\dfrac{G\,B_v}{(1-\beta_2)\epsilon^{3/2}}\right), & \mathrm{AdamW}.
\end{cases}
$

Here AdamW uses the elementwise preconditioner $(v+\epsilon)^{-1/2}$ with $\epsilon>0$.
\end{theorem}

We defer the proof to Appendix~\ref{app:proof_radius_local}. Theorem~\ref{thm:radius_local} shows that optimizer-state mismatch directly \emph{inflates} the local containment radius, increasing the chance of local escape and imposing a stricter geometric requirement on aggregation stability. Notably, the second-moment mismatch term is amplified by the preconditioner: small perturbations in $v$ translate into larger changes in $(v+\epsilon)^{-1/2}$, which motivates prioritizing second-moment synchronization, consistent with \citep{liu2025fedadamw, sun2023efficient_adaptive_fl_second}. Moreover, under bounded heterogeneity (Assumption~\ref{ass:bounded_heterogeneity}), worst-case drift control can be replaced by an average/RMS drift bound, yielding improved \emph{in-expectation} convergence guarantees (Corollary~\ref{cor:expect_convergence} in Appendix).

\section{The Proposed \methodname{} Method}

\begin{figure*}[t]
    \centering
    \begin{subfigure}[b]{0.32\textwidth}
        \centering
        \includegraphics[height=3.8cm, width=\textwidth]{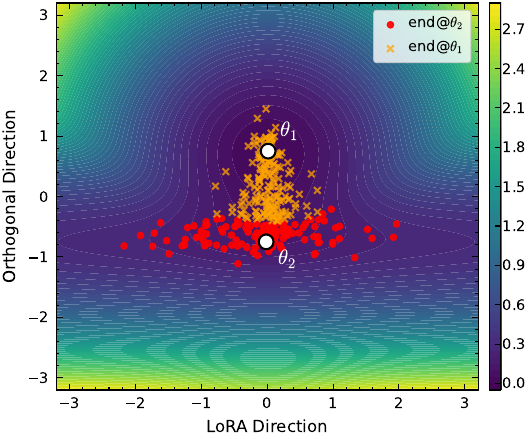}
        \caption{GaLore (Synthetic)}
        \label{fig:syn_galore}
    \end{subfigure}
    \hfill
    \begin{subfigure}[b]{0.32\textwidth}
        \centering
        \includegraphics[height=3.8cm, width=\textwidth]{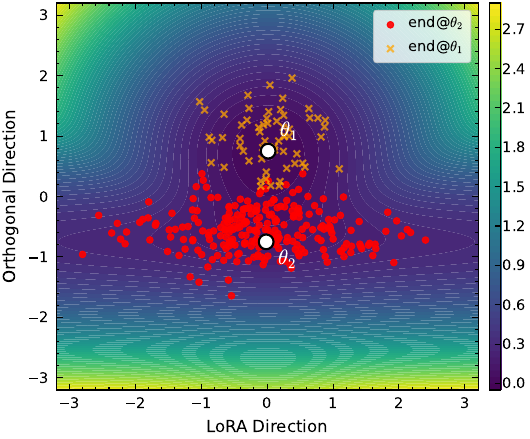}    
        \caption{LoRA (Synthetic)}
        \label{fig:syn_lora}
    \end{subfigure}
    \hfill
    \begin{subfigure}[b]{0.32\textwidth}
        \centering
        \includegraphics[height=3.8cm, width=\textwidth]{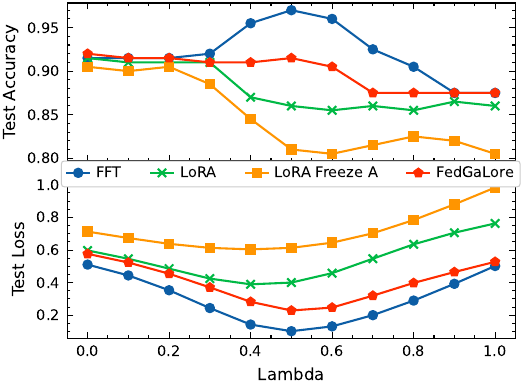}
        \caption{Linear Connectivity (ViT)}
        \label{fig:connectivity}
    \end{subfigure}
    \vspace{-5pt}
\caption{\textbf{Geometric robustness and aggregation stability.}
\textbf{(a--b)} Synthetic landscape: GaLore-style updates reach flat basins more often than fixed-subspace LoRA (60\% vs.\ 20\%; Appendix~\ref{app:synthetic_exp}). 
\textbf{(c)} Non-IID ViT-MNIST: linear interpolation between client models shows low loss barriers for FFT and GaLore-style updates, while LoRA exhibits a higher barrier, indicating reduced aggregation robustness.}
\label{fig:geometric_analysis}
\end{figure*}

Our analysis identifies two coupled failure modes in federated PEFT under heterogeneity:
(1) \emph{Update-space mismatch:} FFT is often aggregation-robust but expensive, whereas LoRA constrains updates to a fixed low-rank manifold, making aggregation brittle under non-IID.
(2) \emph{Optimizer-state mismatch:} adaptive optimizers speed up local training, but unsynchronized second-moment states amplify drift via preconditioning, degrading local containment and aggregation stability.
\methodname{} addresses both by combining (i) client-side \emph{gradient-subspace} optimization with low-rank states and (ii) server-side drift-robust synchronization of second-moment matrices via spectral shared-signal extraction.
Algorithm~\ref{alg:fedgalore} summarizes one round.

\paragraph{Client-side optimization (GaLore on target modules).}
\methodname{} applies GaLore-style \emph{gradient-subspace} optimization to target weight matrices (e.g., attention $\{\mat{W}_Q,\mat{W}_K,\mat{W}_V,\mat{W}_O\}$ and MLP projection matrices), each treated as a matrix block $\mat{W}\in\mathbb{R}^{n\times n}$ for exposition.\footnote{Rectangular blocks are handled analogously. In AdamW, the second-moment state is elementwise; we write $\tilde{\mat{v}}$ for the elementwise buffer reshaped to match the projected tensor.}
At local step $t$, client $i$ maintains a rank-$r$ projector $\mat{P}^{i,k}_t\in\mathbb{R}^{r\times n}$ and computes the dense gradient $\mat{g}^{i,k}_t=\nabla_{\mat{W}}\ell(\theta^{i,k}_t;\mathcal{B}^{i,k}_t)\in\mathbb{R}^{n\times n}$. Under default GaLore (\texttt{proj\_type=std}), we use a one-sided right projection
\[
\tilde{\mat{g}}^{i,k}_t=\mat{g}^{i,k}_t(\mat{P}^{i,k}_t)^\top \in\mathbb{R}^{n\times r},
\]
and maintain projected optimizer buffers $(\tilde{\mat{m}}^{i,k}_t,\tilde{\mat{v}}^{i,k}_t)\in\mathbb{R}^{n\times r}$, reprojecting them whenever $\mat{P}^{i,k}_t$ is refreshed (every $\tau$ steps).

\noindent\textbf{SVD-to-random projectors.}
Following prior work \citep{pan2025unbiased_galore}, we use data-driven projectors (randomized SVD/RSVD) in early local epochs and switch to a random orthonormal projector in later epochs to reduce overhead. In the random phase, the projector is fully determined by an \emph{integer seed}. Accordingly, the server broadcasts a per-round seed $s_k$ enabling clients to deterministically construct the same projector without transmitting basis matrices. We show in Appendix~\ref{app:rgalore_works} that this schedule improves wall-clock efficiency and accelerates local training convergence.

\noindent\textbf{What clients send.}
After $T$ local steps, client $i$ uploads (i) a rank-$r$ factorized model update for the adapted blocks (or equivalently $\Delta\theta_i=\theta^{i,k}_T-\bar{\theta}_k$ in compressed form), and (ii) the projected second-moment buffer $\tilde{\mat{v}}^{i,k}_T\in\mathbb{R}^{n\times r}$ used for server-side synchronization.

This differs from LoRA in a crucial way: LoRA fixes a low-rank \emph{parameterization}, whereas \methodname{} adapts a low-rank \emph{update subspace} estimated from recent gradients. As a result, clients can exploit informative directions outside any fixed low-rank manifold, producing trajectories that are empirically closer to full fine-tuning (FFT) in terms of aggregation robustness \cite{he2024subspace_improve_galore_guarantee}. Figure~\ref{fig:geometric_analysis} supports this intuition: on synthetic landscapes, \methodname{} reaches flatter basins more often than LoRA, and on non-IID ViT-MNIST, it exhibits improved linear connectivity between client models, indicating enhanced aggregation stability.

\begin{algorithm}[t]
\caption{\methodname{} (one communication round $k$; one adapted block $\mat{W}\in\mathbb{R}^{n\times n}$)}
\label{alg:fedgalore}
\begin{algorithmic}[1]\small
\REQUIRE Global model $\bar{\theta}_k$; rank $r$; local steps $T$; projector refresh period $\tau$; client weights $\{\tilde p_i\}$; projection seed $s_k$.
\STATE \textbf{Server:} broadcast $\bar{\theta}_k$ (and seed $s_k$ if using seeded random projectors) to participating clients $\mathcal{P}_k$.
\FORALL{$i\in\mathcal{P}_k$ \textbf{in parallel}}
    \STATE Initialize $\theta^{i,k}_0\leftarrow \bar{\theta}_k$ and GaLoreAdamW state for the adapted block: projected buffers $(\tilde{\mat{m}}^{i,k}_0,\tilde{\mat{v}}^{i,k}_0)\in\mathbb{R}^{n\times r}$ and projector $\mat{P}^{i,k}_0$.
    \FOR{$t=0$ to $T-1$}
        \STATE Compute gradient on the adapted block: $\mat{g}^{i,k}_t=\nabla_{\mat{W}}\ell(\theta^{i,k}_t;\mathcal{B}^{i,k}_t)\in\mathbb{R}^{n\times n}$.
        \STATE GaLoreAdamW step (one-sided projection + buffer reprojection when refreshed every $\tau$ steps):
        $\theta^{i,k}_{t+1}\leftarrow \mathrm{GaLoreAdamW}(\theta^{i,k}_t,\mat{g}^{i,k}_t;\,r,\tau,s_k)$.
    \ENDFOR
    \STATE Client update: 
    $\Delta\theta_i\leftarrow \theta^{i,k}_T-\bar{\theta}_k$ (rank-$r$ factorized).
    \STATE Upload 
    $(\Delta\theta_i,\tilde{\mat{v}}^{i,k}_T)
    $, where $\tilde{\mat{v}}^{i,k}_T\in\mathbb{R}^{n\times r}$ is the projected second-moment buffer for this block.
\ENDFOR
\STATE \textbf{Server:} aggregate parameters (FedAvg):
$\bar{\theta}_{k+1}\leftarrow \bar{\theta}_k+\sum_{i\in\mathcal{P}_k}\tilde p_i\,\Delta\theta_i$.
\STATE Construct matrix views using the shared seed/basis and synchronize second moments:
$\bar{\mat{v}}_{k+1}\leftarrow \mathrm{AJIVE}\big(\{\mat{V}^{i,k}\}_{i\in\mathcal{P}_k}\big)$, where $\mat{V}^{i,k}\coloneqq \tilde{\mat{v}}^{i,k}_T\,\mat{R}_k^\top\in\mathbb{R}^{n\times n}$ and $\mat{R}_k$ is the (right) basis reconstructed from $s_k$.
\STATE Broadcast $(\bar{\theta}_{k+1},\bar{\mat{v}}_{k+1})$ (and next seed $s_{k+1}$ if used).
\end{algorithmic}
\end{algorithm}

\paragraph{Server-side: FedAvg for parameters, AJIVE for second-moment synchronization.}
The server aggregates model parameters by FedAvg,
$\bar{\theta}_{k+1}\leftarrow \bar{\theta}_k+\sum_{i\in\mathcal{P}_k}\tilde p_i\,\Delta\theta_i$.
To mitigate optimizer-state mismatch with minimal communication, \methodname{} synchronizes only the \emph{projected second-moment} buffer (the preconditioner state), which our W.H.P.\ analysis identifies as the most sensitive component under heterogeneity (Theorem~\ref{thm:radius_local}). Concretely, for each adapted block, client $i$ uploads its projected second-moment buffer $\tilde{\mat{v}}^{i,k}_T\in\mathbb{R}^{n\times r}$ (together with the model update). 

\noindent\textbf{Seeded matrix view.}
In the random-adaptive regime, the projector is determined by an integer seed $s_k$ broadcast by the server; all parties deterministically reconstruct the same orthonormal basis $\mat{R}_k\in\mathbb{R}^{n\times r}$ for each block.\footnote{For square blocks under the default right projection, $\mat{R}_k$ is the right basis; rectangular blocks follow the side chosen by GaLore's \texttt{std} rule.}
Then, for each uploaded $\tilde{\mat{v}}^{i,k}_T$ server can project them back to the full space using:
\begin{equation}\label{eq:matrix_view}
\mat{V}^{i,k}\;\coloneqq\;\tilde{\mat{v}}^{i,k}_T\,\mat{R}_k^\top\in\mathbb{R}^{n\times n},
\end{equation}
so each client contributes a rank-$\le r$ within the a shared space $\mathbb{R}^{n \times n}$.

\noindent\textbf{Why AJIVE.}
Under non-IID data, these views contain a shared component (task-level curvature/preconditioning structure) plus client-specific drift and noise. We model
\begin{equation}\label{eq:ajive_decomposition}
\mat{V}^{i,k}=\mat{J}^k+\mat{A}^{i,k}+\mat{E}^{i,k},
\end{equation}
and apply AJIVE \citep{feng2018ajive} across $\{\mat{V}^{i,k}\}_{i\in\mathcal{P}_k}$ to estimate the joint low-rank component $\mat{J}^k$ (Details in Appendix \ref{appendix:ajive_intro}). AJIVE uses a joint-rank parameter $k$ (shared component rank); we set $k=r$ so the synchronized state matches the client projector dimension and can be broadcast as the next-round initializer. Appendix~\ref{appendix:ajive} provides validation experiments showing AJIVE is more robust than naive averaging (even with post-hoc SVD re-projection).

 \paragraph{Computation and communication.}

\begin{table}[h]
\centering
\caption{\textbf{Per-round overhead for one adapted block $\mat{W}\in\mathbb{R}^{n\times n}$.}
Client VRAM counts optimizer buffers for this block only (activations excluded). Uplink counts client$\rightarrow$server payload for this block.}
\label{tab:overhead_block}
\footnotesize
\setlength{\tabcolsep}{3pt}
\renewcommand{\arraystretch}{1.10}
\resizebox{\columnwidth}{!}{%
\begin{tabular}{lcc}
\toprule
& \textbf{Federated LoRA } & \textbf{\methodname{}} \\
\midrule
Client VRAM (states) & $\mathcal{O}(nr)$ & $\mathcal{O}(nr)$ \\
Client extra compute & -- & $\textsc{ProjRefresh}/\tau$ \\
Uplink (per round) & $\mathcal{O}(nr)$ & $\mathcal{O}(nr)$ (update) $+\ \mathcal{O}(nr)$ ($\tilde{\mat{v}}$) \\
\midrule
Server extra compute & -- & $\textsc{AJIVE}(K,n)$ (Table~\ref{tab:ajive_scaling_1024}) \\
\bottomrule
\end{tabular}%
}
\end{table}

We summarize client and server overhead for a single adapted block $\mat{W}\in\mathbb{R}^{n\times n}$ with rank $r$ in Table~\ref{tab:overhead_block}. 
On the client, \methodname{} maintains GaLore-projected optimizer buffers of shape $n\times r$ and refreshes the projector every $\tau$ steps (randomized SVD/RSVD in early epochs, then seeded random projection in later epochs), which amortizes the projection cost over local training.
On the server, AJIVE is applied once per round per adapted block and scales with the number of participating clients (views $=|\mathcal{P}_k|$), not the total client population.
Appendix~\ref{appendix:ajive} benchmarks AJIVE on dense $n\times n$ views up to $n{=}1024$ (Table~\ref{tab:ajive_scaling_1024}); for views$=5$ it adds $\approx$22\,ms on an A100 and $\approx$93\,ms on CPU, which is small relative to per-round training time.

Communication remains LoRA-like: the only additional client uplink beyond Federated LoRA is an $n\times r$ projected second-moment buffer per adapted block (plus a seed/index when using seeded random projectors), which is the same order as transmitting LoRA factors and remains far smaller than dense $n\times n$ states.

\section{Experimental Evaluation}\label{sec:exp}
\begin{table*}[t]
\caption{GLUE Benchmark Results. \ding{51} denotes IID. \ding{53} denotes Non-IID, $\Delta$ indicates the difference. More results in Appendix \ref{app:results_of_glue}. }
\label{tab:glue}
\centering
\small
\label{subtab:glue1}
\begin{tabularx}{\textwidth}{@{}l *{4}{XXX}@{}}
\toprule
\textbf{Method } & 
\multicolumn{3}{c}{\textbf{CoLA} Acc\%} & 
\multicolumn{3}{c}{\textbf{SST} Acc\%} & 
\multicolumn{3}{c}{\textbf{MRPC} Acc\%} & 
\multicolumn{3}{c}{\textbf{QQP} Acc\%} \\
\cmidrule(lr){2-4} \cmidrule(lr){5-7} \cmidrule(lr){8-10} \cmidrule(lr){11-13}
 & {\ding{51}} & {\ding{53}} & {$\Delta$} 
 & {\ding{51}} & {\ding{53}} & {$\Delta$} 
 & {\ding{51}} & {\ding{53}} & {$\Delta$} 
 & {\ding{51}} & {\ding{53}} & {$\Delta$} \\
\midrule
FedAvg-Full  & 83.6 & 83.0 & $\downarrow$0.6 & 94.6 & 94.0 & $\downarrow$0.6 & 89.7& 83.0& $\downarrow$6.7 & 84.4 & 82.1 & $\downarrow$2.3 \\

\midrule

FedIT  & 71.2 & 65.4 & $\downarrow$5.8 & 88.1 & 83.7 & $\downarrow$4.4 & 82.3 & 76.9 & $\downarrow$5.4 & 78.9 & 72.4 & $\downarrow$6.5\\
FFA-LoRA  & 84.0 & 78.3 & $\downarrow$5.7 & 94.1 & 88.4 & $\downarrow$5.7 & 89.3 & 82.1 & $\downarrow$7.2 & 84.7 & 78.2 & $\downarrow$6.5 \\
FLoRA  & \textbf{84.7} & 75.1 & $\downarrow$9.6 & 92.1 & 85.2 & $\downarrow$6.9 & 87.8 & 82.5 & $\downarrow$5.3 & 86.3 & 77.1 & $\downarrow$9.2 \\
LoRA-Fair  & 81.7 & 80.1 & $\downarrow$1.6 & 87.4 & 84.2 & $\downarrow$3.2 & 86.5 & 82.5 & $\downarrow$4.0 & 83.9 & 82.1 & $\downarrow$1.8 \\
FR-LoRA  & 83.9 & 79.4 & $\downarrow$4.5 & 93.7 & 87.6 & $\downarrow$6.1 & 84.3 & 79.2 & $\downarrow$5.1 & \textbf{88.2} & 80.1 & $\downarrow$8.1 \\
\midrule
\methodname{}$^-$  & 83.2 & 81.4 & {$\downarrow$1.8} & \textbf{94.5} & 90.2 & $\downarrow$4.3 & 89.8 & 84.5 & $\downarrow$5.3 & 85.4 & 81.5 & $\downarrow$3.9\\
\methodname{}  & 83.9 & \textbf{83.1} & \textbf{$\downarrow$0.8} & 94.8 & \textbf{93.9} & \textbf{$\downarrow$0.9} & \textbf{89.8} & \textbf{88.2} & \textbf{$\downarrow$1.6} & 84.3 & \textbf{82.5} & \textbf{$\downarrow$1.8}\\
\bottomrule
\end{tabularx}
\end{table*}

\begin{table*}[t]
    \caption{DomainNet Results (excluding Real domain). Accuracies in \%. \ding{51} denotes IID, \ding{53} denotes Non-IID, $\Delta$ indicates the difference.} 
    \label{tab:domainnet_revised}
    \setlength{\tabcolsep}{1.5pt} 
    \small 
    \begin{tabularx}{\textwidth}{@{\hspace{4pt}}l@{\hspace{4pt}} *{5}{>{\centering\arraybackslash}X >{\centering\arraybackslash}X >{\centering\arraybackslash}X}@{}} 
    \toprule
    \begin{tabular}[t]{@{}l@{}}\textbf{Method}\\\end{tabular} & 
    \multicolumn{3}{c}{\textbf{Clipart}} & 
    \multicolumn{3}{c}{\textbf{Painting}} & 
    \multicolumn{3}{c}{\textbf{Infograph}} & 
    \multicolumn{3}{c}{\textbf{Quickdraw}} & 
    \multicolumn{3}{c}{\textbf{Sketch}} \\
    \cmidrule(lr){2-4} \cmidrule(lr){5-7} \cmidrule(lr){8-10} \cmidrule(lr){11-13} \cmidrule(lr){14-16}
    & {\ding{51}} & {\ding{53}} & {$\Delta$} & {\ding{51}} & {\ding{53}} & {$\Delta$} & {\ding{51}} & {\ding{53}} & {$\Delta$} & {\ding{51}} & {\ding{53}} & {$\Delta$} & {\ding{51}} & {\ding{53}} & {$\Delta$} \\
    \midrule
    FedAvg-Full  & 82.3 & 78.9 & $\downarrow$3.4 & 79.2 & 75.8 & $\downarrow$3.4 & 54.3 & 48.5 & $\downarrow$5.8 & 71.0 & 67.2 & $\downarrow$3.8 & 78.3 & 74.1 & $\downarrow$4.2 \\
    \cmidrule{1-16} 
    FedIT  & 80.1 & 74.2 & $\downarrow$5.9 & 77.5 & 70.8 & $\downarrow$6.7 & 52.8 & 45.9 & $\downarrow$6.9 & 69.4 & 62.8 & $\downarrow$6.6 & 76.2 & 69.1 & $\downarrow$7.1 \\
    FFA-LoRA  & 79.8 & 72.5 & $\downarrow$6.3 & 76.9 & 71.2 & $\downarrow$5.7 & 51.9 & 46.2 & $\downarrow$5.7 & 68.7 & 63.5 & $\downarrow$5.2 & 75.8 & 68.4 & $\downarrow$7.4 \\
    FLoRA  & 80.7 & 74.1 & $\downarrow$6.6 & 77.8 & 71.3 & $\downarrow$6.5 & 52.1 & 45.0 & $\downarrow$7.1 & 69.5 & 63.2 & $\downarrow$6.3 & 76.0 & 68.5 & $\downarrow$7.5 \\

    LoRA-Fair  & 82.3 & 81.1 & $\downarrow$\textbf{1.2} & 73.2 & 70.3 & $\downarrow$2.9 & 43.8 & 41.0 & $\downarrow$\textbf{2.8} & 63.3 & 59.1 & $\downarrow$4.2 & 73.4 & 68.5 & $\downarrow$4.9 \\
    
    FR-LoRA  & 81.7 & 77.9 & $\downarrow$3.8 & 78.2 & 74.1 & $\downarrow$4.1 & 52.1 & 43.3 & $\downarrow$8.8 & 70.2 & 64.6 & $\downarrow$5.6 & 77.1 & 69.1 & $\downarrow$8.0 \\

    \midrule
    \methodname{}$^-$ & 82.1 & 71.3 & $\downarrow$10.8 & \textbf{79.5} & 72.4 & $\downarrow$7.1 & \textbf{53.8} & 41.9 & $\downarrow$11.9 & 69.2 & 64.3 & $\downarrow$4.9 & 76.8 & 70.5 & $\downarrow$ 6.3 \\
    \methodname{} & \textbf{81.9} & \textbf{78.2} & $\downarrow$3.7 & 78.8 & \textbf{75.1} & $\downarrow$3.7 & \textbf{53.8} & \textbf{47.6} & $\downarrow$6.2 & \textbf{70.5} & \textbf{66.4} & $\downarrow$\textbf{4.1} & \textbf{77.9} & \textbf{73.5} & $\downarrow$\textbf{4.4} \\
    \bottomrule
    \end{tabularx}\label{tab:domainnet}
\end{table*}

We evaluate \methodname{} on three domains:
\textbf{NLU:} RoBERTa-base \citep{liu2019roberta} on 7 GLUE tasks \citep{wang2018glue};
\textbf{Vision:} ViT-base \citep{dosovitskiy2020vit} on 6 DomainNet domains \citep{peng2019domainnet};
\textbf{NLG:} Llama-2-7B \citep{touvron2023llama} fine-tuned on MetaMathQA \citep{yu2023metamath} and evaluated on GSM8K \citep{cobbe2021gsm8k} and MATH \citep{hendrycksmath}.

\textbf{Baselines.} We compare to FedAvg-Full (full fine-tuning\footnote{For Llama-2-7B, ``full fine-tuning'' refers to updating the full weight matrices of the target modules where LoRA/GaLore is applied (not the entire model).}), FedIT \citep{zhang2024towards}, FFA-LoRA \citep{sun2024improving}, FLoRA \citep{wang2024flora}, FR-LoRA \citep{yanfrlora}, and LoRA-Fair \citep{bian2024lora-fair}. We use each baseline’s original optimizer choice when available, and otherwise match learning rate, rank, and target modules.  We also report \methodname$^{-}$ (GaLore in FL without drift-robust state synchronization) as an ablation. All PEFT methods use the same target modules and rank.

\textbf{Federated protocol.} For NLU and vision, we use $M{=}50$ clients with partial participation ($K{=}5$/round, sampled uniformly). For NLG, we use $M{=}4$ clients with full participation ($K{=}4$) due to higher per-client cost. All methods use the same training budget (rounds and local steps). Hyperparameters and implementation details are in Appendix~\ref{appendix:hyperparameters}.

\begin{table}[t]
\caption{LLM Federated Fine-Tuning Results. Accuracies in \% on GSM8K  and Math. \ding{51} denotes IID, \ding{53} denotes Non-IID, $\Delta$ indicates the difference. Higher is better.}
\label{tab:llm_two_tasks}
\setlength{\tabcolsep}{2pt}
\small
\begin{tabularx}{\linewidth}{@{\hspace{2pt}}l@{\hspace{2pt}} *{2}{>{\centering\arraybackslash}X >{\centering\arraybackslash}X >{\centering\arraybackslash}X}@{}}
\toprule
\begin{tabular}[t]{@{}l@{}}\textbf{Method}\\\end{tabular} &
\multicolumn{3}{c}{\textbf{GSM8K}} &
\multicolumn{3}{c}{\textbf{MATH}} \\
\cmidrule(lr){2-4}\cmidrule(lr){5-7}
& {\ding{51}} & {\ding{53}} & {$\Delta$} 
& {\ding{51}} & {\ding{53}} & {$\Delta$} \\
\midrule
FedAvg-Full & 35.0 & 32.3 & $\downarrow$2.7 & 4.5  & 4.1 & $\downarrow$ 0.4 \\
\cmidrule{1-7}
FedIT     &  31.3 &  28.2 & $\downarrow$3.1 & 4.7  & 4.0  & $\downarrow$ 0.7  \\
FFA-LoRA  & 28.0  &  26.2& $\downarrow$1.8  & 3.8 & 3.8 & \textbf{$\downarrow$ 0.0}  \\
LoRA-Fair &  29.3 &  28.5 & \textbf{$\downarrow$ 0.8} &  4.1&  3.8 & $\downarrow$ 0.3 \\
FLoRA       & 39.8 &  37.2& $\downarrow$ 2.6 & 4.6 &  4.0 & $\downarrow$ 0.6 \\
FR-LoRA   & 44.3  & 39.2 & $\downarrow$ 5.1  & 5.2  & 4.3 & $\downarrow$ 0.9  \\
\midrule
\methodname{}$^-$   & 41.7 &  38.5 & $\downarrow$3.2  & 5.6 & 4.9 & $\downarrow$ 0.7  \\
\methodname{} & \textbf{44.9} & \textbf{41.9} & $\downarrow$2.7 & \textbf{5.7} & \textbf{5.1} & $\downarrow$ 0.6 \\
\bottomrule
\end{tabularx}
\end{table}

\paragraph{Results and discussion.}
We report IID  and non-IID ($\alpha=0.5$) performance and use their gap as a simple sensitivity metric: smaller $\Delta$ indicates lower sensitivity to non-IIDness.

\textbf{NLU (GLUE) and vision (DomainNet).}
Across GLUE (Table~\ref{tab:glue}) and DomainNet (Table~\ref{tab:domainnet}), federated LoRA baselines generally show much larger robustness gaps than FedAvg-Full, whose minimal degradation suggests that full-space optimization/aggregation can remain stable under these shifts. Several LoRA variants incur large drops (e.g., CoLA and multiple DomainNet domains), consistent with \emph{update-space mismatch}: clients adapt in restricted low-rank directions that become misaligned across clients, and aggregation amplifies off-subspace components. Methods that constrain the update space more aggressively (e.g., FFA-LoRA, LoRA-Fair) reduce $\Delta$ in some settings but at the cost of lower IID accuracy, indicating reduced expressivity. \methodname{} improves this trade-off, matching or improving IID performance while maintaining consistently small $\Delta$ across tasks/domains. The ablation \methodname$^{-}$ boosts IID accuracy but collapses under data heterogeneity (notably on DomainNet), showing that dynamic subspaces alone are insufficient---\emph{drift-robust state synchronization} is essential.

\textbf{NLG (LLM fine-tuning).}
Robustness gaps are smaller on GSM8K and MATH (Table~\ref{tab:llm_two_tasks}), suggesting the pretrained Llama-2 backbone provides stronger invariances. Local-adaptive methods (FLoRA/FR-LoRA/\methodname$^{-}$) also achieve higher IID performance than static LoRA baselines, indicating that reasoning tasks benefit from greater optimization plasticity. However, local adaptivity increases sensitivity when optimizer states drift across clients. \methodname{} retains strong IID performance while reducing the non-IID drop, consistent with combining gradient-subspace exploration with drift-robust second-moment synchronization.

\textbf{Practical takeaway.}
FedAvg-Full is the most robust but costly. Federated LoRA is efficient, but can be brittle. \methodname{} bridges this gap, delivering near-FFT robustness with PEFT-level efficiency, especially when local adaptivity is needed.

\section{Conclusions and Future Work}
We present \methodname{}, a robust protocol for federated parameter-efficient fine-tuning under non-IID data. Our W.H.P. analysis attributes federated LoRA brittleness to two coupled mismatches—update-space and optimizer-state—and motivates combining GaLore-style gradient-subspace optimization with drift-robust synchronization of projected second-moment statistics, delivering consistent gains across NLU, vision, and LLM benchmarks. Beyond improved accuracy, \methodname{} yields a practical stability–efficiency trade-off: near full fine-tuning robustness without transmitting dense optimizer states, enabled by low-rank projected synchronization and lightweight server-side filtering. Future work includes extending the analysis to broader participation regimes and develop a tighter end-to-end W.H.P. guarantees.

\section*{Impact Statement}

This paper presents work whose goal is to advance the field of Machine
Learning. There are many potential societal consequences of our work, none
which we feel must be specifically highlighted here.

\bibliography{references}
\bibliographystyle{icml2026}

\newpage
\appendix
\onecolumn

\section{Details of Local Training Operator $\mathcal{T}$}\label{appendix:local_training_operator}

\begin{algorithm}[h]
\caption{Local Training Operator $\mathcal{T}$ (Stochastic Gradient Descent)}
\label{alg:local-training-sgd}
\begin{algorithmic}[1]
\STATE \textbf{Input:} global model $\bar{\theta}^k$, steps $T$, stepsize $\eta$
\STATE \textbf{Output:} local model $\theta^{i,k}_T$

\STATE Initialize $\theta^{i,k}_0 \gets \bar{\theta}^k$

\FOR{$t = 0$ \textbf{to} $T-1$}
    \STATE Sample minibatch $\mathcal{B}^{i,k}_t \sim \mathcal D_i$
    \STATE Compute stochastic gradient
    $
        g^{i,k}_t = \nabla \ell(\theta^{i,k}_t;\,\mathcal{B}^{i,k}_t)
    $

    \STATE Model update:
    $
        \theta^{i,k}_{t+1}
        =
        \theta^{i,k}_t
        -
        \eta\, g^{i, k}_t
    $
\ENDFOR

\STATE Send $\theta^{i,k}_T$ to the server
\end{algorithmic}
\end{algorithm}

\begin{algorithm}[h]
\caption{Local Training Operator $\mathcal{T}$ (SGD with Momentum)}
\label{alg:local-training-sgdm}
\begin{algorithmic}[1]
\STATE \textbf{Input:} global model $\bar{\theta}^k$, local optimizer state $S^{i,k}_0$, steps $T$, stepsize $\eta$, momentum factor $\mu$
\STATE \textbf{Output:} local model $\theta^{i,k}_T$, local state $S^{i,k}_T$

\STATE Initialize $\theta^{i,k}_0 \gets \bar{\theta}^k$
\STATE Initialize momentum buffer $v^{i,k}_0 \gets S^{i,k}_0$ \COMMENT{Initialize to $\mat{0}$ if $k=0$}

\FOR{$t = 0$ \textbf{to} $T-1$}
    \STATE Sample minibatch $\mathcal{B}^{i,k}_t \sim \mathcal D_i$
    \STATE Compute stochastic gradient
    $
        g^{i,k}_t = \nabla \ell(\theta^{i,k}_t;\,\mathcal{B}^{i,k}_t)
    $

    \STATE Update momentum buffer:
    $
        v^{i,k}_{t+1}
        =
        \mu\, v^{i,k}_t
        +
        g^{i, k}_t
    $

    \STATE Model update:
    $
        \theta^{i,k}_{t+1}
        =
        \theta^{i,k}_t
        -
        \eta\, v^{i,k}_{t+1}
    $
\ENDFOR

\STATE Set $S^{i,k}_T = v^{i,k}_T$
\STATE Send $(\theta^{i,k}_T,\ S^{i,k}_T)$ to the server
\end{algorithmic}
\end{algorithm}

\begin{algorithm}[!h]
\caption{Local Training Operator $\mathcal{T}$ (Adam \cite{kingma2014adam}/ AdamW \cite{loshchilov2017adamw})}
\label{alg:local-training-adamw}
\begin{algorithmic}[1]
\STATE \textbf{Input:} global model $\bar{\theta}^k$, local optimizer state $S^{i,k}_0$, steps $T$, stepsize $\eta$, optimizer $\mathcal{O}$, momentum parameters $(\beta_1,\beta_2)$, numerical constant $\varepsilon$
\STATE \textbf{Output:} local model $\theta^{i,k}_T$, local state $S^{i,k}_T$

\STATE Initialize $\theta^{i,k}_0 \gets \bar{\theta}^k$
\STATE Initialize first and second momentum $m^{i,k}_0, v^{i,k}_0 \gets S^{i,k}_0 $

\FOR{$t = 0$ \textbf{to} $T-1$}
    \STATE Sample minibatch $\mathcal{B}^{i,k}_t \sim \mathcal D_i$
    \STATE Compute stochastic gradient
    $
        g^{i,k}_t = \nabla \ell(\theta^{i,k}_t;\,\mathcal{B}^{i,k}_t)
    $

    \STATE Update first momentum:
        $m^{i,k}_{t+1}
        =
        \beta_1 m^{i,k}_t
        +
        (1-\beta_1) g^{i, k}_t$

    \STATE Update second moment:
    $
        v^{i,k}_{t+1}
        =
        \beta_2 v^{i,k}_t
        +
        (1-\beta_2)\, g^{i, k}_t \odot g^{i, k}_t
    $

    \STATE Model update:
    $
        \theta^{i,k}_{t+1}
        =
        \theta^{i,k}_t
        -
        \eta
        \frac{
             m^{i,k}_{t+1}
        }{
            \sqrt{v^{i,k}_{t+1}} + \varepsilon
        }
    $

\ENDFOR

\STATE Set $S^{i,k}_T = ( m^{i,k}_T,  v^{i,k}_T)$;
\STATE Send 
 $(\theta^{i,k}_T,\ S^{i,k}_T)$ to the server
\end{algorithmic}
\end{algorithm}

\subsection{Details of GaLore-AdamW}
\label{app:galore_adamw}

This section summarizes the GaLoreAdamW update used by \methodname{} and clarifies the tensor shapes under the \textbf{default} setting (\texttt{proj\_type=std}, one-sided projection) for an adapted weight block $\mat{W}\in\mathbb{R}^{m\times n}$ (square case $m{=}n$ is a special case).

\paragraph{Default projection rule and shapes.}
GaLore maintains a rank-$r$ projector that is refreshed every \texttt{update\_proj\_gap} steps (denoted $\tau$ in the paper). Under \texttt{proj\_type=std}, the projector chooses the side based on the aspect ratio:
\begin{itemize}[leftmargin=*,noitemsep,topsep=0pt]
\item If $m\ge n$ (including square blocks), use a \textbf{right} basis $\mat{P}\in\mathbb{R}^{r\times n}$ and project
$\tilde{\mat{g}}=\mat{g}\mat{P}^\top\in\mathbb{R}^{m\times r}$.
\item If $m<n$, use a \textbf{left} basis $\mat{P}\in\mathbb{R}^{m\times r}$ and project
$\tilde{\mat{g}}=\mat{P}^\top\mat{g}\in\mathbb{R}^{r\times n}$.
\end{itemize}
For our exposition with square attention/MLP blocks $\mat{W}\in\mathbb{R}^{n\times n}$, the default is the \textbf{right} case, hence $\mat{P}\in\mathbb{R}^{r\times n}$ and $\tilde{\mat{g}}\in\mathbb{R}^{n\times r}$.

\paragraph{Projected AdamW state (no lifting).}
GaLoreAdamW stores the AdamW buffers in the \emph{same shape as the projected gradient}:
\[
\tilde{\mat{m}}_t \in \mathbb{R}^{m\times r},\quad
\tilde{\mat{v}}_t \in \mathbb{R}^{m\times r}\qquad (\text{right projection, }m\ge n),
\]
(and analogously $\tilde{\mat{m}}_t,\tilde{\mat{v}}_t\in\mathbb{R}^{r\times n}$ for left projection).
Here $\tilde{\mat{v}}_t$ denotes the \emph{elementwise} second-moment buffer reshaped to the projected tensor shape. Importantly, the buffers are updated in this projected space without forming any dense $m\times n$ second-moment matrix.

\paragraph{GaLoreAdamW update (right projection case).}
Let $\mat{g}_t=\nabla_{\mat{W}}\ell(\theta;\mathcal{B}_t)\in\mathbb{R}^{m\times n}$ be the dense gradient for the block and $\tilde{\mat{g}}_t=\mat{g}_t\mat{P}_t^\top\in\mathbb{R}^{m\times r}$ its projected version. GaLoreAdamW performs standard AdamW updates in the projected shape:
\begin{align*}
\tilde{\mat{m}}_{t} &\leftarrow \beta_1 \tilde{\mat{m}}_{t-1} + (1-\beta_1)\tilde{\mat{g}}_t,\\
\tilde{\mat{v}}_{t} &\leftarrow \beta_2 \tilde{\mat{v}}_{t-1} + (1-\beta_2)(\tilde{\mat{g}}_t\odot \tilde{\mat{g}}_t),\\
\tilde{\mat{d}}_t &\leftarrow \sqrt{\tilde{\mat{v}}_t}+\epsilon, \qquad
\tilde{\mat{u}}_t \leftarrow \tilde{\mat{m}}_t \oslash \tilde{\mat{d}}_t,
\end{align*}
optionally with bias correction (controlled by \texttt{correct\_bias}). The projected, preconditioned update $\tilde{\mat{u}}_t\in\mathbb{R}^{m\times r}$ is then \emph{projected back} to the ambient shape:
\[
\mat{u}_t \;=\; \tilde{\mat{u}}_t\,\mat{P}_t \in \mathbb{R}^{m\times n},
\]
and the parameter update is applied in the ambient space:
\[
\mat{W}\leftarrow \mat{W} - \eta_t\,\mat{u}_t \;-\; \eta_t\,\lambda\,\mat{W},
\]
where $\eta_t$ includes the AdamW step size (with bias correction when enabled) and $\lambda$ is the decoupled weight decay.

\paragraph{Projector refresh and buffer reprojection (stay low-rank).}
When the projector is refreshed (every $\tau$ steps), GaLoreAdamW reprojects the stored buffers to remain consistent with the new basis, without lifting to $m\times n$ in the common right$\rightarrow$right case.
For right projection, let $\mat{P}_{\text{old}},\mat{P}_{\text{new}}\in\mathbb{R}^{r\times n}$. Define their corresponding column-space bases $\mat{V}_{\text{old}}=\mat{P}_{\text{old}}^\top\in\mathbb{R}^{n\times r}$ and $\mat{V}_{\text{new}}=\mat{P}_{\text{new}}^\top\in\mathbb{R}^{n\times r}$. GaLoreAdamW applies the low-rank change-of-basis transform
\[
\tilde{\mat{m}} \leftarrow \tilde{\mat{m}}\,(\mat{V}_{\text{old}}^\top \mat{V}_{\text{new}}), \qquad
\tilde{\mat{v}} \leftarrow \tilde{\mat{v}}\,(\mat{V}_{\text{old}}^\top \mat{V}_{\text{new}}),
\]
which maps the $m\times r$ buffers to the new projected coordinates. For square blocks under \texttt{proj\_type=std}, the side remains ``right'' across refreshes, so this low-rank reprojection avoids materializing dense $m\times n$ tensors. (A rare fallback path lifts and reprojects only if the projection \emph{side} changes.)

\paragraph{Seeded random-adaptive projectors.}
Under \texttt{proj\_type=random-adaptive} or \texttt{rsvd-adaptive}, the projector uses data-driven bases for the first $S$ refreshes (\texttt{adaptive\_proj\_steps}), then switches to a seeded random orthonormal basis. In the random phase, the basis is fully determined by an integer seed; thus synchronizing the projector across clients requires broadcasting only the seed (not the basis matrix), while the projected buffers $\tilde{\mat{m}},\tilde{\mat{v}}$ remain in the same $m\times r$ (or $r\times n$) shape and are reprojected as above when the seed increments.

\paragraph{Summary of shapes (square block, default).}
For $\mat{W}\in\mathbb{R}^{n\times n}$ and default \texttt{proj\_type=std} (right projection):
\[
\mat{P}\in\mathbb{R}^{r\times n},\quad
\tilde{\mat{g}},\tilde{\mat{m}},\tilde{\mat{v}},\tilde{\mat{u}}\in\mathbb{R}^{n\times r},\quad
\mat{u}\in\mathbb{R}^{n\times n}.
\]
Thus, GaLoreAdamW maintains adaptive states in $\mathcal{O}(nr)$ memory per adapted block and applies dense parameter updates via \texttt{project\_back} without storing dense second-moment matrices.

\newpage

\section{Proofs} \label{app:proofs_state}

\subsection{Conditional sub-Gaussian noise and a W.H.P.\ envelope}\label{app:noise_env}

\begin{definition}[Conditional sub-Gaussian vector]\label{def:subg}
A random vector $X\in\mathbb{R}^d$ is \emph{conditionally $\sigma$-sub-Gaussian given $\mathcal{F}$} if
$\mathbb{E}[X\mid\mathcal{F}]=0$ and for all $\lambda\in\mathbb{R}$ and all unit vectors $u\in\mathbb{R}^d$,
\[
\mathbb{E}\!\left[\exp\!\big(\lambda\langle u,X\rangle\big)\,\middle|\,\mathcal{F}\right]
\le \exp\!\left(\frac{\lambda^2\sigma^2}{2}\right).
\]
\end{definition}

\begin{lemma}[W.H.P.\ noise envelope]\label{lem:noise_env}
Assume the mini-batch noise $\xi_t^{i,k}$ is conditionally mean-zero and $\sigma$-sub-Gaussian given $\mathcal{F}_{k,t-1}$
in the sense of Definition~\ref{def:subg}. Then for any $\delta\in(0,1)$, with probability at least $1-\delta$,
\begin{equation}\label{eq:noise_env_app}
\max_{0\le k\le K-1}\max_{0\le t\le T-1}\max_{i\in[M]}\|\xi_t^{i,k}\|_2
\;\le\;
\varepsilon_{\mathrm{noise}}(\delta)
\;\coloneqq\;
\sigma\sqrt{2d\log\!\Big(\frac{2dMKT}{\delta}\Big)}.
\end{equation}
\end{lemma}

\begin{proof}
Fix indices $(i,k,t)$ and a coordinate $j\in[d]$. By Definition~\ref{def:subg} with $u=e_j$, the scalar
$Z\coloneqq [\xi_t^{i,k}]_j$ is conditionally $\sigma$-sub-Gaussian given $\mathcal{F}_{k,t-1}$, hence the standard tail bound
holds conditionally:
\[
\mathbb{P}\!\left(|Z|\ge u \,\middle|\, \mathcal{F}_{k,t-1}\right)\le 2\exp\!\left(-\frac{u^2}{2\sigma^2}\right)\quad \forall u\ge 0.
\]
Taking expectations over $\mathcal{F}_{k,t-1}$ yields the same unconditional bound
$\mathbb{P}(|Z|\ge u)\le 2\exp(-u^2/(2\sigma^2))$.

Union bound over all $(i,k,t,j)\in [M]\times\{0,\dots,K-1\}\times\{0,\dots,T-1\}\times[d]$ gives
\[
\mathbb{P}\!\left(\max_{i,k,t}\|\xi_t^{i,k}\|_\infty \ge u\right)
\le 2dMKT\exp\!\left(-\frac{u^2}{2\sigma^2}\right).
\]
Set $u=\sigma\sqrt{2\log(2dMKT/\delta)}$ so the right-hand side is at most $\delta$.
On this event, for all $(i,k,t)$,
\[
\|\xi_t^{i,k}\|_2 \le \sqrt{d}\,\|\xi_t^{i,k}\|_\infty \le \sqrt{d}\,u
= \sigma\sqrt{2d\log\!\Big(\frac{2dMKT}{\delta}\Big)}.
\]
This proves~\eqref{eq:noise_env_app}.
\end{proof}

\subsection{Proof of Theorem~\ref{thm:radius_local}}\label{app:proof_radius_local}

Throughout this proof, we work on the event
\[
\mathcal{E}_{\mathrm{noise}}(\delta)\;\coloneqq\;
\left\{\max_{k,t,i}\|\xi_t^{i,k}\|_2 \le \varepsilon_{\mathrm{noise}}(\delta)\right\},
\]
which holds with probability at least $1-\delta$ by Lemma~\ref{lem:noise_env}.
For readability, write $\varepsilon\coloneqq \varepsilon_{\mathrm{noise}}(\delta)$.

We fix a round $k$ and suppress the superscript $k$; the bounds are uniform over $k$ because
$\mathcal{E}_{\mathrm{noise}}(\delta)$ already union-bounds over all rounds.

\paragraph{Setup.}
Fix a client $i$. Let $\{\theta_t\}_{t=0}^T$ be the client trajectory and $\{\theta_t^\star\}_{t=0}^T$ the reference trajectory,
with common initialization $\theta_0=\theta_0^\star=\bar\theta^k$.
Define deviations
\[
\Delta_t \coloneqq \theta_t-\theta_t^\star,\qquad \Delta_0=0.
\]
Assume $f=\frac1M\sum_{i=1}^M F_i$ is $L$-smooth on $Q$ and gradients are clipped/bounded on $Q$ so
$\|\nabla F_i(\theta)\|_2\le G$ for all $i$ and $\theta\in Q$. Then $\|\nabla f(\theta)\|_2\le G$ and
\begin{equation}\label{eq:drift_2G}
\|\nabla F_i(\theta)-\nabla f(\theta)\|_2 \le 2G,\qquad \forall \theta\in Q.
\end{equation}
We also assume the reference iterates $\theta_t^\star\in Q$ and the client iterates remain in $Q$ on the success event;
all inequalities below are conditioned on being in $Q$ where smoothness/boundedness holds.

\subsubsection*{Case 1: SGD}

The client update is
\[
\theta_{t+1}=\theta_t-\eta\big(\nabla F_i(\theta_t)+\xi_t\big),
\qquad
\theta^\star_{t+1}=\theta^\star_t-\eta\nabla f(\theta^\star_t).
\]
Subtracting yields
\begin{align}
\Delta_{t+1}
&=\Delta_t-\eta\big(\nabla f(\theta_t)-\nabla f(\theta^\star_t)\big)
-\eta\big(\nabla F_i(\theta_t)-\nabla f(\theta_t)\big)-\eta\xi_t. \label{eq:sgd_delta}
\end{align}
Taking norms and using $L$-smoothness of $f$ and~\eqref{eq:drift_2G}, on $\mathcal{E}_{\mathrm{noise}}(\delta)$ we obtain
\[
\|\Delta_{t+1}\|_2 \le (1+\eta L)\|\Delta_t\|_2 + \eta(2G+\varepsilon).
\]
Unrolling from $\Delta_0=0$ gives
\[
\|\Delta_t\|_2
\le \eta(2G+\varepsilon)\sum_{s=0}^{t-1}(1+\eta L)^{t-1-s}
=
\eta(2G+\varepsilon)\frac{(1+\eta L)^t-1}{\eta L}.
\]
Using $(1+\eta L)^t\le e^{\eta Lt}$ and $e^x-1\le xe^x$, we get
\[
\|\Delta_t\|_2 \le (2G+\varepsilon)\cdot \frac{e^{\eta Lt}-1}{L}
\le (2G+\varepsilon)\cdot \frac{\eta Lt\,e^{\eta Lt}}{L}
= \eta t(2G+\varepsilon)e^{\eta Lt}.
\]
Under the step-size condition $\eta LT\le \tfrac12$, we have $e^{\eta Lt}\le e^{1/2}\le 2$ for all $t\le T$, hence
\begin{equation}\label{eq:sgd_final}
\max_{t\le T}\|\Delta_t\|_2 \le 2\eta T(2G+\varepsilon) = 4\eta TG + 2\eta T\varepsilon.
\end{equation}
This matches the stated $\mathcal{R}_{\mathrm{drift}}$ for SGD up to absolute constants.

\subsubsection*{Case 2: Momentum}

Momentum recursions (one round) are
\[
m_{t+1}=\beta_1 m_t+(1-\beta_1)\big(\nabla F_i(\theta_t)+\xi_t\big),
\qquad
\theta_{t+1}=\theta_t-\eta m_{t+1},
\]
and for the reference
\[
m_{t+1}^\star=\beta_1 m_t^\star+(1-\beta_1)\nabla f(\theta_t^\star),
\qquad
\theta_{t+1}^\star=\theta_t^\star-\eta m_{t+1}^\star.
\]
Define $\Delta^m_t\coloneqq m_t-m_t^\star$ and note $\|\Delta^m_0\|_2=\|m_0^{i,k}-m_0^{\star,k}\|_2\le B_m$
by Definition~\ref{def:state_bias_compact}. Subtracting the $m$-updates gives
\begin{equation}\label{eq:mom_dm}
\Delta^m_{t+1}
=\beta_1\Delta^m_t+(1-\beta_1)\Big(\nabla F_i(\theta_t)-\nabla f(\theta_t^\star)+\xi_t\Big).
\end{equation}
Decompose $\nabla F_i(\theta_t)-\nabla f(\theta_t^\star)
=\big(\nabla f(\theta_t)-\nabla f(\theta_t^\star)\big)+\big(\nabla F_i(\theta_t)-\nabla f(\theta_t)\big)$ and use
$L$-smoothness plus~\eqref{eq:drift_2G} and $\|\xi_t\|_2\le\varepsilon$ on $\mathcal{E}_{\mathrm{noise}}(\delta)$ to obtain
\begin{equation}\label{eq:mom_dm_bound}
\|\Delta^m_{t+1}\|_2
\le \beta_1\|\Delta^m_t\|_2+(1-\beta_1)\big(L\|\Delta_t\|_2+2G+\varepsilon\big).
\end{equation}
Now express $\Delta_t$ in terms of $\Delta^m$:
\[
\Delta_{t+1}=\Delta_t-\eta\Delta^m_{t+1}
\quad\Rightarrow\quad
\Delta_t=-\eta\sum_{r=0}^{t-1}\Delta^m_{r+1}.
\]
Using the closed-form solution of~\eqref{eq:mom_dm},
\[
\Delta^m_{r+1}=\beta_1^{r+1}\Delta^m_0 + (1-\beta_1)\sum_{s=0}^{r}\beta_1^{r-s} h_s,
\qquad
h_s\coloneqq \nabla F_i(\theta_s)-\nabla f(\theta_s^\star)+\xi_s,
\]
we obtain (by summation and swapping sums) the identity
\[
\Delta_t
=
-\eta\sum_{r=0}^{t-1}\beta_1^{r+1}\Delta^m_0
-\eta\sum_{s=0}^{t-1} w_{t,s}\,h_s,
\qquad
w_{t,s}\coloneqq (1-\beta_1)\sum_{r=s}^{t-1}\beta_1^{r-s}=1-\beta_1^{t-s}\in[0,1].
\]
Therefore,
\begin{equation}\label{eq:mom_delta_convolution}
\|\Delta_t\|_2
\le
\underbrace{\eta\sum_{r=0}^{t-1}\beta_1^{r+1}\|\Delta^m_0\|_2}_{\le \eta B_m/(1-\beta_1)}
+\eta\sum_{s=0}^{t-1}\|h_s\|_2.
\end{equation}
Finally, bound $\|h_s\|_2$ as
\[
\|h_s\|_2
\le
\|\nabla f(\theta_s)-\nabla f(\theta_s^\star)\|_2+\|\nabla F_i(\theta_s)-\nabla f(\theta_s)\|_2+\|\xi_s\|_2
\le L\|\Delta_s\|_2+2G+\varepsilon.
\]
Plugging into~\eqref{eq:mom_delta_convolution} yields, for all $t\le T$,
\begin{equation}\label{eq:mom_gronwall}
\|\Delta_t\|_2
\le
\frac{\eta}{1-\beta_1}B_m
+\eta t(2G+\varepsilon)
+\eta L\sum_{s=0}^{t-1}\|\Delta_s\|_2.
\end{equation}
We apply the standard discrete Gr\"onwall inequality: if $a_t\ge 0$ satisfies
$a_t \le A + Bt + C\sum_{s=0}^{t-1} a_s$, then $a_t\le (A+Bt)e^{Ct}$.
Applying this to $a_t=\|\Delta_t\|_2$ with $A=\frac{\eta}{1-\beta_1}B_m$, $B=\eta(2G+\varepsilon)$, $C=\eta L$ gives
\[
\|\Delta_t\|_2 \le \left(\frac{\eta}{1-\beta_1}B_m+\eta t(2G+\varepsilon)\right)e^{\eta Lt}.
\]
Under $\eta LT\le \tfrac12$, we have $e^{\eta Lt}\le 2$ for $t\le T$, hence
\begin{equation}\label{eq:mom_final}
\max_{t\le T}\|\Delta_t\|_2
\le
2\cdot \frac{\eta}{1-\beta_1}B_m + 2\eta T(2G+\varepsilon)
=
\frac{2\eta}{1-\beta_1}B_m + 4\eta TG + 2\eta T\varepsilon.
\end{equation}
This matches the theorem’s momentum bound up to absolute constants.

\subsubsection*{Case 3: AdamW with preconditioner $(v+\epsilon\mathbf{1})^{-1/2}$}

We analyze the AdamW-style update (bias corrections and weight decay omitted for clarity; they only affect constants).
Let $g_t=\nabla F_i(\theta_t)+\xi_t$ (client) and $g_t^\star=\nabla f(\theta_t^\star)$ (reference), with recursions
\[
m_{t+1}=\beta_1 m_t+(1-\beta_1)g_t,\quad
v_{t+1}=\beta_2 v_t+(1-\beta_2)g_t^{\odot 2},\quad
\theta_{t+1}=\theta_t-\eta\,m_{t+1}\odot(v_{t+1}+\epsilon\mathbf{1})^{-1/2},
\]
and similarly for $(m_t^\star,v_t^\star,\theta_t^\star)$ with $g_t^\star$.

\paragraph{Step 1: a sensitivity bound for the preconditioner.}
For scalars $a,b\ge 0$,
\begin{equation}\label{eq:inv_sqrt_lip}
\left|\frac{1}{\sqrt{a+\epsilon}}-\frac{1}{\sqrt{b+\epsilon}}\right|
\le \frac{|a-b|}{2\epsilon^{3/2}},
\qquad
\left|\frac{1}{\sqrt{a+\epsilon}}-\frac{1}{\sqrt{b+\epsilon}}\right|
\le \frac{1}{\sqrt{\epsilon}}.
\end{equation}
The first inequality follows from the mean value theorem and $|(d/dx)(x+\epsilon)^{-1/2}|\le (2\epsilon^{3/2})^{-1}$, while
the second holds since $(a+\epsilon)^{-1/2}\in(0,\epsilon^{-1/2}]$.

\paragraph{Step 2: bounding the update-direction difference.}
Define the update directions
\[
u_{t+1}\coloneqq m_{t+1}\odot(v_{t+1}+\epsilon\mathbf{1})^{-1/2},
\qquad
u_{t+1}^\star\coloneqq m_{t+1}^\star\odot(v_{t+1}^\star+\epsilon\mathbf{1})^{-1/2}.
\]
Then $\Delta_{t+1}=\Delta_t-\eta(u_{t+1}-u_{t+1}^\star)$, so
\[
\|\Delta_t\|_2 \le \eta\sum_{s=0}^{t-1}\|u_{s+1}-u_{s+1}^\star\|_2.
\]
Add and subtract $m_{t+1}^\star\odot(v_{t+1}+\epsilon)^{-1/2}$ to get
\[
u_{t+1}-u_{t+1}^\star
=
(m_{t+1}-m_{t+1}^\star)\odot(v_{t+1}+\epsilon)^{-1/2}
+
m_{t+1}^\star\odot\left((v_{t+1}+\epsilon)^{-1/2}-(v_{t+1}^\star+\epsilon)^{-1/2}\right).
\]
Using $\|a\odot b\|_2\le \|a\|_2\|b\|_\infty$ and $\|(v_{t+1}+\epsilon)^{-1/2}\|_\infty\le \epsilon^{-1/2}$ yields
\begin{equation}\label{eq:u_diff_basic}
\|u_{t+1}-u_{t+1}^\star\|_2
\le
\frac{1}{\sqrt{\epsilon}}\|m_{t+1}-m_{t+1}^\star\|_2
+
\|m_{t+1}^\star\|_2\cdot \left\|(v_{t+1}+\epsilon)^{-1/2}-(v_{t+1}^\star+\epsilon)^{-1/2}\right\|_\infty.
\end{equation}

On $\mathcal{E}_{\mathrm{noise}}(\delta)$, $\|g_t\|_2\le G+\varepsilon$ and $\|g_t^\star\|_2\le G$.
By induction on the EMA recursions (starting from bounded states at round start), we have $\|m_{t}^\star\|_2\le G$ for all $t$.
(If $m_0^\star=0$ this is immediate; more generally it holds whenever $m_0^\star$ is produced by the same EMA under bounded gradients.)

\paragraph{Step 3: isolating the effect of \emph{initial} state mismatch.}
Let $\Delta_0^m\coloneqq m_0-m_0^\star$ and $\Delta_0^v\coloneqq v_0-v_0^\star$, with
$\|\Delta_0^m\|_2\le B_m$ and $\|\Delta_0^v\|_2\le B_v$.

\emph{(i) First moment.}
Unrolling the $m$-recursion gives
\[
m_{t+1}-m_{t+1}^\star
=
\beta_1^{t+1}\Delta_0^m
+
(1-\beta_1)\sum_{s=0}^{t}\beta_1^{t-s}(g_s-g_s^\star).
\]
On $\mathcal{E}_{\mathrm{noise}}(\delta)$, $\|g_s-g_s^\star\|_2\le \|g_s\|_2+\|g_s^\star\|_2\le 2G+\varepsilon$,
and $(1-\beta_1)\sum_{s=0}^t\beta_1^{t-s}\le 1$, hence
\begin{equation}\label{eq:m_diff_bound}
\|m_{t+1}-m_{t+1}^\star\|_2
\le \beta_1^{t+1}B_m + (2G+\varepsilon).
\end{equation}

\emph{(ii) Second moment.}
Similarly,
\[
v_{t+1}-v_{t+1}^\star
=
\beta_2^{t+1}\Delta_0^v
+
(1-\beta_2)\sum_{s=0}^{t}\beta_2^{t-s}\big(g_s^{\odot 2}-(g_s^\star)^{\odot 2}\big).
\]
We bound the preconditioner difference by separating the contribution of $\beta_2^{t+1}\Delta_0^v$ from the remaining term.
Let $r_{t+1}\coloneqq (1-\beta_2)\sum_{s=0}^{t}\beta_2^{t-s}\big(g_s^{\odot 2}-(g_s^\star)^{\odot 2}\big)$ so that
$v_{t+1}=v_{t+1}^\star+\beta_2^{t+1}\Delta_0^v+r_{t+1}$.
Then by triangle inequality and~\eqref{eq:inv_sqrt_lip},
\begin{align}
\left\|(v_{t+1}+\epsilon)^{-1/2}-(v_{t+1}^\star+\epsilon)^{-1/2}\right\|_\infty
&\le
\left\|(v_{t+1}^\star+r_{t+1}+\beta_2^{t+1}\Delta_0^v+\epsilon)^{-1/2}-(v_{t+1}^\star+r_{t+1}+\epsilon)^{-1/2}\right\|_\infty \nonumber\\
&\quad+
\left\|(v_{t+1}^\star+r_{t+1}+\epsilon)^{-1/2}-(v_{t+1}^\star+\epsilon)^{-1/2}\right\|_\infty \nonumber\\
&\le
\frac{\beta_2^{t+1}\|\Delta_0^v\|_\infty}{2\epsilon^{3/2}}+\frac{1}{\sqrt{\epsilon}}
\;\le\;
\frac{\beta_2^{t+1}B_v}{2\epsilon^{3/2}}+\frac{1}{\sqrt{\epsilon}}. \label{eq:precond_split}
\end{align}
(The last step uses $\|\Delta_0^v\|_\infty\le \|\Delta_0^v\|_2\le B_v$.)

\paragraph{Step 4: complete the bound.}
Combine~\eqref{eq:u_diff_basic}, $\|m_{t+1}^\star\|_2\le G$,~\eqref{eq:m_diff_bound}, and~\eqref{eq:precond_split} to obtain, for all $t$,
\[
\|u_{t+1}-u_{t+1}^\star\|_2
\le
\frac{1}{\sqrt{\epsilon}}\big(\beta_1^{t+1}B_m + (2G+\varepsilon)\big)
+
G\left(\frac{\beta_2^{t+1}B_v}{2\epsilon^{3/2}}+\frac{1}{\sqrt{\epsilon}}\right).
\]
Summing over $t=0,\dots,T-1$ yields, for all $t\le T$,
\begin{align*}
\|\Delta_t\|_2
&\le \eta\sum_{s=0}^{t-1}\|u_{s+1}-u_{s+1}^\star\|_2\\
&\le
\frac{\eta}{\sqrt{\epsilon}}\left(\sum_{s=0}^{t-1}\beta_1^{s+1}\right)B_m
+
\frac{\eta G}{2\epsilon^{3/2}}\left(\sum_{s=0}^{t-1}\beta_2^{s+1}\right)B_v
+
\frac{\eta t}{\sqrt{\epsilon}}\Big( (2G+\varepsilon)+G\Big)\\
&\le
\frac{\eta}{(1-\beta_1)\sqrt{\epsilon}}\,B_m
+
\frac{\eta G}{2(1-\beta_2)\epsilon^{3/2}}\,B_v
+
\frac{\eta T}{\sqrt{\epsilon}}\big(3G+\varepsilon\big).
\end{align*}
Thus,
\begin{equation}\label{eq:adamw_final}
\max_{t\le T}\|\Delta_t\|_2
\le
\frac{\eta}{(1-\beta_1)\sqrt{\epsilon}}\,B_m
+
\frac{\eta G}{2(1-\beta_2)\epsilon^{3/2}}\,B_v
+
\frac{\eta T}{\sqrt{\epsilon}}\big(3G+\varepsilon\big).
\end{equation}
This matches the theorem’s AdamW decomposition into a drift/noise term of order $\frac{\eta T}{\sqrt{\epsilon}}(G+\varepsilon)$ and a
state term of order $\eta\left(\frac{B_m}{(1-\beta_1)\sqrt{\epsilon}}+\frac{G B_v}{(1-\beta_2)\epsilon^{3/2}}\right)$,
up to absolute constants.

\paragraph{Conclusion.}
Combining~\eqref{eq:sgd_final}, \eqref{eq:mom_final}, and \eqref{eq:adamw_final} on $\mathcal{E}_{\mathrm{noise}}(\delta)$,
and noting that Lemma~\ref{lem:noise_env} already provides a simultaneous bound over all $(i,k,t)$, yields the claimed W.H.P.\ tube bound
simultaneously over all rounds, clients, and steps.

\subsection{In-expectation refinement under bounded heterogeneity}\label{app:expect_cor}

The W.H.P.\ bound in Theorem~\ref{thm:radius_local} uses worst-case drift control induced by clipping (via $\|\nabla F_i-\nabla f\|\le 2G$).
Under the assumption of bounded heterogeneity \ref{ass:bounded_heterogeneity}, one can obtain tighter \emph{average/RMS} drift bounds and corresponding in-expectation stability/convergence results.
We state one representative corollary for SGD (extensions to stateful optimizers follow the same template but are more notational).

\begin{corollary}[In-expectation stability and convergence under bounded heterogeneity]\label{cor:expect_convergence}
Assume full participation with server averaging and local \textnormal{SGD} for $T$ steps per round:
$\bar\theta^{k+1}=\frac{1}{M}\sum_{i=1}^M \theta_T^{i,k}$.
Assume on $Q$: (i) $f$ is $L$-smooth and satisfies the PL inequality with parameter $\mu>0$; (ii) bounded heterogeneity holds:
\[
\frac1M\sum_{i=1}^M\|\nabla F_i(\theta)\|_2^2 \le H^2 + B^2\|\nabla f(\theta)\|_2^2,\quad \forall \theta\in Q;
\]
(iii) mini-batch noise is unbiased with bounded conditional second moment:
$\mathbb{E}[\xi_t^{i,k}\mid \mathcal{F}_{k,t-1}]=0$ and $\mathbb{E}[\|\xi_t^{i,k}\|_2^2\mid \mathcal{F}_{k,t-1}]\le \sigma^2$;
and (iv) $\|\nabla f(\theta)\|_2\le G$ on $Q$.
Let the within-round reference be $T$ steps of (full-batch) gradient descent on $f$:
$\theta_{t+1}^{\star,k}=\theta_{t}^{\star,k}-\eta\nabla f(\theta_t^{\star,k})$ with $\theta_0^{\star,k}=\bar\theta^k$.
If $\eta LT\le \frac{1}{6}$, then for each round $k$,
\[
\mathbb{E}\Big[\|\bar\theta^{k+1}-\theta_T^{\star,k}\|_2\ \big|\ \bar\theta^k\Big]
\;\le\;
2\eta T\left(\sqrt{H^2+(B^2-1)G^2}+\sigma\right).
\]
Moreover, defining $r_k\coloneqq \mathbb{E}[f(\bar\theta^k)-f^\star]$, there exist absolute constants $c_1,c_2>0$ such that
\[
r_{k+1}
\le
(1-\eta\mu)^T\,r_k
+
c_1\,\eta T\left(\sqrt{H^2+(B^2-1)G^2}+\sigma\right)
+
c_2\,\eta^2T^2\left(\sqrt{H^2+(B^2-1)G^2}+\sigma\right)^2.
\]
In particular, if $H=\sigma=0$ (and hence clients are effectively IID and gradients are deterministic), then $r_k$ converges
\emph{linearly} to $0$ (i.e., $\bar\theta^k$ converges in expectation to $f^\star$). More generally, with constant $\eta$ the recursion
converges linearly to a neighborhood whose radius scales with $H$ and $\sigma$, and with diminishing $\eta\to 0$ the neighborhood vanishes.
\end{corollary}

\begin{proof}
Fix a round $k$ and omit the superscript $k$ for readability.
Let $\Delta_t^i\coloneqq \theta_t^{i}-\theta_t^\star$ and define the RMS deviation
\[
D_t \coloneqq \left(\frac1M\sum_{i=1}^M\|\Delta_t^i\|_2^2\right)^{1/2},\qquad D_0=0.
\]

\paragraph{Step 1: an RMS recursion using bounded heterogeneity.}
For SGD, $\theta_{t+1}^i=\theta_t^i-\eta(\nabla F_i(\theta_t^i)+\xi_t^i)$ and
$\theta_{t+1}^\star=\theta_t^\star-\eta\nabla f(\theta_t^\star)$, so
\[
\Delta_{t+1}^i=\Delta_t^i-\eta\big(\nabla f(\theta_t^i)-\nabla f(\theta_t^\star)\big)
-\eta\big(\nabla F_i(\theta_t^i)-\nabla f(\theta_t^i)\big)-\eta\xi_t^i.
\]
Stack the vectors across clients: $\widetilde\Delta_t\in\mathbb{R}^{Md}$ collects $(\Delta_t^1,\dots,\Delta_t^M)$, and similarly define
stacked $\widetilde A_t$ for $(\nabla f(\theta_t^i)-\nabla f(\theta_t^\star))_{i=1}^M$,
$\widetilde B_t$ for $(\nabla F_i(\theta_t^i)-\nabla f(\theta_t^i))_{i=1}^M$, and $\widetilde\xi_t$ for $(\xi_t^i)_{i=1}^M$.
Then $\widetilde\Delta_{t+1}=\widetilde\Delta_t-\eta\widetilde A_t-\eta\widetilde B_t-\eta\widetilde\xi_t$ and by the triangle inequality
in $\mathbb{R}^{Md}$,
\begin{equation}\label{eq:rms_tri}
\|\widetilde\Delta_{t+1}\|_2
\le
\|\widetilde\Delta_t-\eta\widetilde A_t\|_2 + \eta\|\widetilde B_t\|_2 + \eta\|\widetilde\xi_t\|_2.
\end{equation}
Since $f$ is $L$-smooth, $\|\nabla f(\theta_t^i)-\nabla f(\theta_t^\star)\|_2\le L\|\Delta_t^i\|_2$ for each $i$, hence
$\|\widetilde A_t\|_2\le L\|\widetilde\Delta_t\|_2$, and therefore
$\|\widetilde\Delta_t-\eta\widetilde A_t\|_2\le (1+\eta L)\|\widetilde\Delta_t\|_2$.
Dividing~\eqref{eq:rms_tri} by $\sqrt{M}$ gives
\begin{equation}\label{eq:D_basic}
D_{t+1}\le (1+\eta L)D_t + \eta\,\underbrace{\left(\frac1M\sum_{i=1}^M\|\nabla F_i(\theta_t^i)-\nabla f(\theta_t^i)\|_2^2\right)^{1/2}}_{=:B_t}
+\eta\,\underbrace{\left(\frac1M\sum_{i=1}^M\|\xi_t^i\|_2^2\right)^{1/2}}_{=:N_t}.
\end{equation}

We bound $B_t$ by relating drift at $\theta_t^i$ to drift at the common reference point $\theta_t^\star$:
\[
\nabla F_i(\theta_t^i)-\nabla f(\theta_t^i)
=
\underbrace{\nabla F_i(\theta_t^\star)-\nabla f(\theta_t^\star)}_{\delta_i(\theta_t^\star)}
+
\big(\nabla F_i(\theta_t^i)-\nabla F_i(\theta_t^\star)\big)
+
\big(\nabla f(\theta_t^\star)-\nabla f(\theta_t^i)\big).
\]
Taking RMS over $i$ and using the triangle inequality plus $L$-smoothness of $F_i$ and $f$ yields
\begin{equation}\label{eq:B_t_bound}
B_t \le \left(\frac1M\sum_{i=1}^M\|\delta_i(\theta_t^\star)\|_2^2\right)^{1/2} + 2L D_t.
\end{equation}
Now compute the RMS drift at a common point using bounded heterogeneity:
\[
\frac1M\sum_{i=1}^M\|\delta_i(\theta)\|_2^2
=
\frac1M\sum_{i=1}^M\|\nabla F_i(\theta)\|_2^2 - \|\nabla f(\theta)\|_2^2
\le
H^2 + (B^2-1)\|\nabla f(\theta)\|_2^2
\le
H^2 + (B^2-1)G^2.
\]
Therefore $\left(\frac1M\sum_i\|\delta_i(\theta_t^\star)\|_2^2\right)^{1/2}\le \sqrt{H^2+(B^2-1)G^2}$.
For the noise term, by Jensen and the conditional variance bound,
\[
\mathbb{E}[N_t\mid \mathcal{F}_{t-1}]
\le \sqrt{\mathbb{E}[N_t^2\mid\mathcal{F}_{t-1}]}
=
\sqrt{\frac1M\sum_{i=1}^M\mathbb{E}[\|\xi_t^i\|_2^2\mid\mathcal{F}_{t-1}]}
\le \sigma.
\]
Taking conditional expectations in~\eqref{eq:D_basic} and using~\eqref{eq:B_t_bound} gives
\[
\mathbb{E}[D_{t+1}\mid \mathcal{F}_{t-1}]
\le (1+\eta L)D_t + \eta\big(\sqrt{H^2+(B^2-1)G^2}+2L D_t\big) + \eta\sigma
=
(1+3\eta L)D_t + \eta\sqrt{H^2+(B^2-1)G^2} + \eta\sigma.
\]

\paragraph{Step 2: unrolling the RMS recursion.}
Iterating from $D_0=0$ yields
\[
\mathbb{E}[D_T\mid \bar\theta^k]
\le
\eta\sum_{t=0}^{T-1}(1+3\eta L)^{T-1-t}\Big(\sqrt{H^2+(B^2-1)G^2}+\sigma\Big)
\le
\eta T\,(1+3\eta L)^T\Big(\sqrt{H^2+(B^2-1)G^2}+\sigma\Big).
\]
Under $\eta LT\le \frac{1}{6}$, we have $(1+3\eta L)^T\le e^{3\eta LT}\le e^{1/2}\le 2$, hence
\[
\mathbb{E}[D_T\mid \bar\theta^k]\le 2\eta T\Big(\sqrt{H^2+(B^2-1)G^2}+\sigma\Big).
\]
Since $\bar\theta^{k+1}-\theta_T^\star = \frac1M\sum_i(\theta_T^i-\theta_T^\star)$, we have
$\|\bar\theta^{k+1}-\theta_T^\star\|_2 \le \frac1M\sum_i\|\Delta_T^i\|_2 \le D_T$.
This proves the first inequality in the corollary.

\paragraph{Step 3: transferring reference contraction to the global iterate.}
By Assumption \ref{ass:pl-condition} and Assumption \ref{ass:smoothness}, gradient descent on $f$ with step size $\eta\le 1/L$ satisfies the standard linear rate
\[
f(\theta_T^\star)-f^\star \le (1-\eta\mu)^T\big(f(\bar\theta^k)-f^\star \big).
\]
Next, by $L$-smoothness of $f$,
\[
f(\bar\theta^{k+1})
\le
f(\theta_T^\star) + \langle \nabla f(\theta_T^\star),\bar\theta^{k+1}-\theta_T^\star\rangle
+ \frac{L}{2}\|\bar\theta^{k+1}-\theta_T^\star\|_2^2.
\]
Using $\|\nabla f(\theta)\|_2\le G$ on $Q$ and $e_k\coloneqq \bar\theta^{k+1}-\theta_T^\star$,
\[
f(\bar\theta^{k+1})-f^\star
\le
f(\theta_T^\star)-f^\star + G\|e_k\|_2 + \frac{L}{2}\|e_k\|_2^2.
\]
Taking expectations and using $\|e_k\|\le D_T$ and the bound on $\mathbb{E}[D_T\mid \bar\theta^k]$ yields
\[
r_{k+1}
\le
(1-\eta\mu)^T r_k
+
G\,\mathbb{E}[D_T]
+
\frac{L}{2}\mathbb{E}[D_T^2].
\]
Finally, use Jensen to bound $\mathbb{E}[D_T^2]\le (\mathbb{E}[D_T])^2$ and substitute the bound on $\mathbb{E}[D_T]$.
This gives the claimed recursion with explicit constants (absorbed into $c_1,c_2$).
The special cases ($H=\sigma=0$ and/or diminishing stepsizes) follow immediately.
\end{proof}

\section{Details of Synthetic Loss Landscape Experiments}
\label{app:synthetic_exp}

To empirically validate our intuition that GaLore tends to find solutions with better loss landscape geometry, we designed a controlled synthetic experiment. The goal is to isolate how the \emph{dynamics} of low-rank optimizers influence their ability to escape "sharp, aligned" local minima in favor of "flat, orthogonal" global basins.

\subsection{Why Flatness implies Aggregation Robustness}
In Federated Learning, the geometric property of \enquote{flatness} is a direct proxy for aggregability. We seek a global model $\mat{W}^\star$ that remains valid even when averaged with perturbed local models. If client models drift by $\Delta \mat{W}_i$, the expected loss of the aggregate is approximately:
\begin{equation}
    \mathbb{E}[\mathcal{L}(\mat{W}_{\mathrm{agg}})] \approx \mathcal{L}(\mat{W}^\star) + \frac{1}{2} \text{Tr}\left( \nabla^2 \mathcal{L}(\mat{W}^\star) \mathbb{E}[\Delta \mat{W} \Delta \mat{W}^\top] \right).
\end{equation}
In a \textbf{sharp valley}, the Hessian $\nabla^2 \mathcal{L}$ has large eigenvalues, so even small drift incurs a massive penalty. In a \textbf{flat basin}, eigenvalues are small, minimizing the impact of heterogeneity.

\subsection{Landscape Construction: The "Kinetic Trap"}
We optimize a matrix $\mat{W} \in \mathbb{R}^{d \times d}$. The landscape combines two quadratic basins via a smooth SoftMin function:
\begin{equation}
    \mathcal{L}(\mat{W}) = -\tau \log\left( \exp(-\mathcal{L}_1(\mat{W})/\tau) + \exp(-\mathcal{L}_2(\mat{W})/\tau) \right).
\end{equation}
We define two key directions relative to the initialization $\mat{W}_0$:
\begin{itemize}
    \item $\mathbf{e}_1$ (\textbf{Aligned}): A direction lying within the initial LoRA subspace span($\mat{A}_0^\top$).
    \item $\mathbf{e}_2$ (\textbf{Orthogonal}): A direction orthogonal to the initialization.
\end{itemize}

The basins are constructed to create a specific geometric challenge:
\begin{itemize}
    \item \textbf{Basin 1 (The Target):} A \textbf{flat basin} centered at $\mat{W}_1 = c \cdot \mathbf{e}_2$. This minimum is robust (small eigenvalues in all directions) but requires significant rotation into the orthogonal direction $\mathbf{e}_2$ to be discovered.
    \item \textbf{Basin 2 (The Trap):} A \textbf{sharp valley} centered at the origin but elongated along $\mathbf{e}_1$. We set the Hessian $\mat{H}_2$ such that the valley is steep in orthogonal directions but shallow along $\mathbf{e}_1$. This creates a \enquote{trap}: LoRA's dynamics, which prioritize movement in the current subspace, perceive a quick path of descent into this brittle minimum.
\end{itemize}

\subsection{Optimization Setup}
We compare three methods, initializing each trial from a randomized starting point $\mat{W}^{(0)} = \mat{W}_{\mathrm{ref}} + \mat{\Xi}$, where $\mat{W}_{\mathrm{ref}}$ is a fixed reference point between the two basins and $\mat{\Xi}$ is Gaussian noise drawn for each trial. 

\begin{enumerate}
    \item \textbf{Full-Space SGD:} Updates $\mat{W}_{t+1} \leftarrow \mat{W}_t - \alpha \nabla \mathcal{L}$.
    \item \textbf{LoRA:} Updates $\mat{W} = \mat{W}_0 + \mat{B}\mat{A}$. Although the subspace is not strictly fixed, the coupling between $\mat{B}$ and $\mat{A}$ creates an \enquote{inertial alignment} that resists rapid rotation into $\mathbf{e}_2$.
    \item \textbf{GaLore:} Explicitly projects gradients onto a time-varying subspace $\mat{U}_r \mat{\Sigma}_r \mat{V}_r^\top$ derived from the SVD of the gradient, allowing for active subspace rotation.
\end{enumerate}

\subsection{Results}

Table~\ref{tab:synthetic_results} confirms that LoRA is attracted to the "aligned" sharp valley because its natural dynamics are insufficient to traverse the orthogonal gap. GaLore's active rotation allows it to escape the trap and converge to the flat basin, matching the robustness of Full SGD.

\begin{table}[h]
\centering
\caption{Probability of converging to the Flat Basin (Robust) vs. Sharp Valley (Brittle).}
\begin{tabular}{lccc}
\toprule
\textbf{Method} & \textbf{Basin 1 (Flat/Robust)} & \textbf{Basin 2 (Sharp/Brittle)} & \textbf{Mechanism} \\
\midrule
Full SGD & \textbf{91\%} & 9\% & Unconstrained \\
GaLore & \textbf{60\%} & 40\% & Active Rotation \\
LoRA & 20\% & \textbf{80\%} & Inertial Alignment \\
\bottomrule
\end{tabular}
\label{tab:synthetic_results}
\end{table}

\section{SVD-to-random projectors: efficiency and local convergence}
\label{app:rgalore_works}

\noindent\textbf{Motivation.}
GaLore-style methods require periodically refreshing a rank-$r$ projector. A purely data-driven refresh (SVD/RSVD every $\tau$ steps) improves early optimization but incurs nontrivial overhead, while purely random projection has lower projector-estimation overhead but may start from a poorly aligned subspace. Following prior work on unbiased/efficient GaLore variants \citep{pan2025unbiased_galore}, we adopt an \emph{SVD-to-random} schedule: 
In our implementation, we use SVD/RSVD for the first $S$ projector refreshes (set by \texttt{adaptive\_proj\_steps}) and then switch to seeded random projectors thereafter. In the random phase, the projector is fully determined by an integer seed, so we broadcast only the seed (not the basis matrix), enabling deterministic reconstruction across clients.

\noindent\textbf{Variants.}
We compare three projector schedules:
(i) \textbf{FedGaLore} (ours): RSVD/SVD in early epochs $\rightarrow$ seeded random thereafter;
(ii) \textbf{FedGaLore-SVD}: always RSVD/SVD refresh;
(iii) \textbf{FedGaLore-Pure Random}: always seeded random projector.
We also plot representative federated LoRA baselines (FFA-LoRA-AdamW, FLoRA-AdamW) for reference.

\begin{figure}[h]
    \centering
    \begin{subfigure}[t]{0.49\columnwidth}
        \centering
        \includegraphics[width=\linewidth]{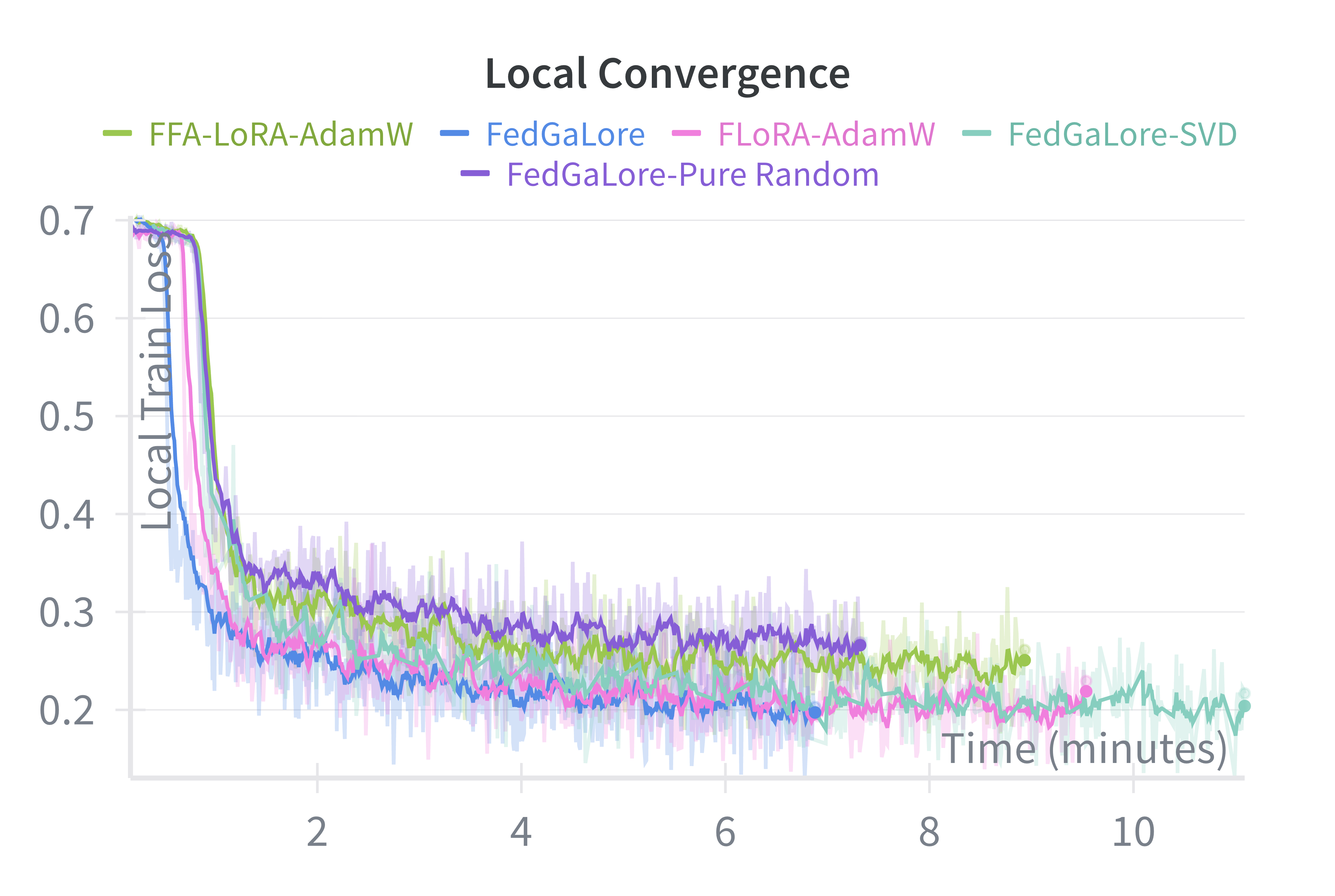}
        \caption{RoBERTa (GLUE): local training loss vs wall-clock time.}
        \label{fig:rgalore_roberta}
    \end{subfigure}
    \hfill
    \begin{subfigure}[t]{0.49\columnwidth}
        \centering
        \includegraphics[width=\linewidth]{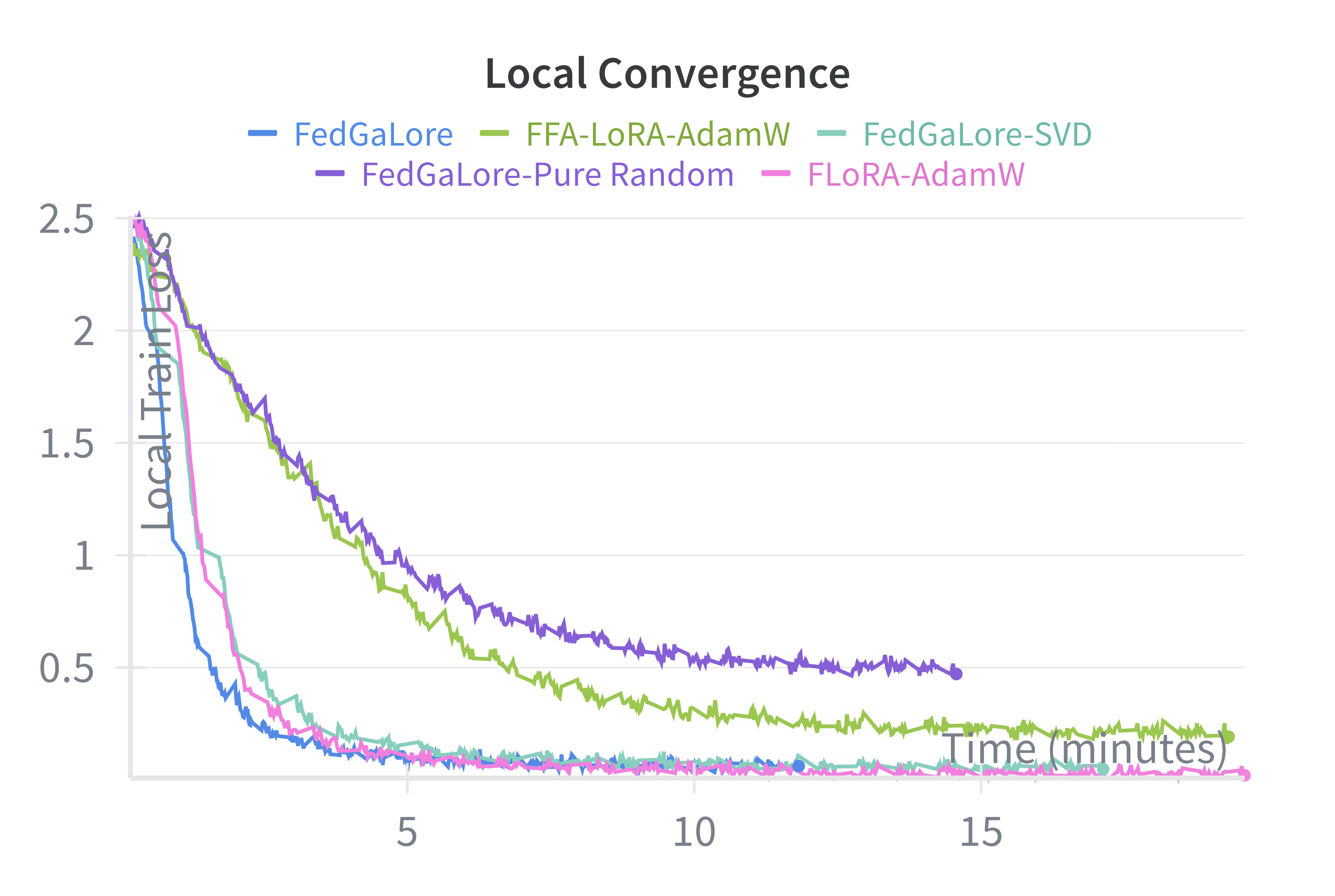}
        \caption{ViT (DomainNet): local training loss vs wall-clock time.}
        \label{fig:rgalore_vit}
    \end{subfigure}
    \vspace{-4pt}
    \caption{\textbf{Local convergence under different projector schedules.}
FedGaLore (SVD$\rightarrow$random) achieves the best time-to-loss compared to always-SVD and always-random schedules.}
    \label{fig:rgalore_local_convergence}
    \vspace{-6pt}
\end{figure}

\noindent\textbf{Results and discussion (time-to-loss, same steps).}
Figure~\ref{fig:rgalore_local_convergence} reports \emph{local training loss versus wall-clock time} under a fixed local training budget: all variants run the same number of local steps with identical hyperparameters (rank, $\tau$, learning rate, batch size, etc.). Differences therefore reflect \emph{time-per-step} (including projector refresh/reprojection overhead) and optimization efficiency per step, rather than additional steps.
Across both RoBERTa and ViT, \textbf{FedGaLore} (SVD$\rightarrow$random) achieves the best time-to-loss. Compared to \textbf{FedGaLore-SVD}, this is expected since repeatedly refreshing RSVD/SVD bases and reprojecting buffers incurs nontrivial overhead. \textbf{FedGaLore-Pure Random} is consistently worse in time-to-loss, indicating that fully random projectors can slow local optimization in practice. While the exact wall-clock advantage over pure random can depend on implementation details (e.g., orthonormalization kernels, caching, and refresh schedules), the empirical takeaway is clear: a short data-driven warm start followed by seeded random projectors yields the most favorable time-to-loss.

\noindent\textbf{Comparison to federated LoRA baselines.}
Figure~\ref{fig:rgalore_local_convergence} also includes representative LoRA baselines under the same local-step budget. We observe that LoRA with AdamW can achieve fast local loss reduction, consistent with the benefit of local adaptivity. In contrast, FFA-LoRA is typically slower locally in these settings which is consistent with its more restrictive update parameterization that trades expressivity for stability under heterogeneity. Notably, these local curves do not capture global robustness: our main results show that methods with strong local progress can still suffer from optimizer-state and update-space mismatch under non-IID aggregation, motivating \methodname{}'s server-side state synchronization.

\newpage
\section{Angle-based Joint and Individual Variation Explained (AJIVE)}\label{appendix:ajive_intro}

\begin{table}[!h]
\centering
\begin{threeparttable}
\caption{\textbf{AJIVE scalability microbenchmark (server-side).}
We measure AJIVE latency as a function of the number of views (participating clients per round, views $=|\mathcal{P}_k|$) and matrix size for dense square inputs $v\in\mathbb{R}^{n\times n}$ (avg over trials). GPU uses an A100 with randomized SVD. FLOPs are estimates.}
\label{tab:ajive_scaling_1024}
\begin{tabular}{lrrrr}
\toprule
Device & views ($|\mathcal{P}_k|$) & $n$ & Average time (s) & Average FLOPs (est.) \\
\midrule
CPU & 1  & 512  & 0.006554 & 4.375e+06 \\
CPU & 1  & 768  & 0.009816 & 9.708e+06 \\
CPU & 1  & 1024 & 0.018628 & 1.714e+07 \\
CPU & 2  & 512  & 0.012233 & 8.849e+06 \\
CPU & 2  & 768  & 0.025425 & 1.956e+07 \\
CPU & 2  & 1024 & 0.040607 & 3.447e+07 \\
CPU & 5  & 512  & 0.031390 & 2.193e+07 \\
CPU & 5  & 768  & 0.045956 & 4.862e+07 \\
CPU & 5  & 1024 & 0.092544 & 8.579e+07 \\
CPU & 10 & 512  & 0.060522 & 4.372e+07 \\
CPU & 10 & 768  & 0.098389 & 9.704e+07 \\
CPU & 10 & 1024 & 0.184439 & 1.713e+08 \\
\midrule
GPU (A100) & 1  & 512  & 0.031459 & 4.375e+06 \\
GPU (A100) & 1  & 768  & 0.004896 & 9.708e+06 \\
GPU (A100) & 1  & 1024 & 0.004826 & 1.714e+07 \\
GPU (A100)\tnote{1} & 2  & 512  & 0.858201 & 8.849e+06 \\
GPU (A100) & 2  & 768  & 0.009561 & 1.956e+07 \\
GPU (A100) & 2  & 1024 & 0.009093 & 3.447e+07 \\
GPU (A100) & 5  & 512  & 0.867758 & 2.193e+07 \\
GPU (A100) & 5  & 768  & 0.023258 & 4.862e+07 \\
GPU (A100) & 5  & 1024 & 0.022056 & 8.579e+07 \\
GPU (A100) & 10 & 512  & 0.747936 & 4.372e+07 \\
GPU (A100) & 10 & 768  & 0.046743 & 9.704e+07 \\
GPU (A100) & 10 & 1024 & 0.045271 & 1.713e+08 \\
\bottomrule
\end{tabular}
\begin{tablenotes}
    \item[1] GPU timings include synchronization and may exhibit variance due to kernel launch/allocator effects. This may lead to abnormal execution time. 
\end{tablenotes}
\end{threeparttable}
\end{table}

We provide a brief introduction to the Joint and Individual Variation Explained (JIVE) framework and the Angle-based JIVE (AJIVE) method. The decomposition of multi-view data into shared and specific components is a fundamental challenge that has sparked diverse applications, such as identifying communities within heterogeneous networks \citep{macdonald2022latent-ajive}, analyzing high-dimensional genomic data \citep{lock2013joint-ajive}, and discerning global and local features in federated learning \citep{shi2024personalized-ajive}.

A significant challenge in these analyses lies in separating the joint and individual variations present in multi-view datasets. In their seminal work, \citet{lock2013joint-ajive} pioneered the JIVE model to address this decomposition. To overcome computational bottlenecks and identifiability issues inherent in the original formulation, \citet{feng2018ajive} proposed AJIVE, a two-stage spectral method based on principal angle analysis. Recent theoretical work by \citet{yang2025ajive-theorem} establishes that AJIVE achieves optimal performance in high-signal-to-noise ratio (SNR) settings, confirming its efficiency while outlining specific limitations; we refer interested readers to this work for details.

Our implementation follows the logic of the MVLearn library \citep{perry2021mvlearn}. For computational efficiency, we have reimplemented the method using PyTorch \citep{pytorch} to leverage GPU acceleration. Algorithm \ref{alg:ajive} details the AJIVE procedure.

\begin{algorithm}[!h]
\caption{Angle-based Joint and Individual Variation Explained (AJIVE)}
\label{alg:ajive}
\begin{algorithmic}[1]
\STATE \textbf{Input:} Data matrices $\{\mat{X}^{(1)}, \dots, \mat{X}^{(k)}\}$, parameters for rank estimation
\STATE \textbf{Output:} Joint matrices $\mat{J}^{(i)}$, Individual matrices $\mat{I}^{(i)}$, Noise matrices $\mat{E}^{(i)}$

\STATE \textit{// Phase 1: Signal Space Initial Extraction}
\FOR{$i = 1$ to $k$}
    \STATE Center $\mat{X}^{(i)}$ by subtracting column means
    \STATE Estimate initial signal rank $r_{\text{init}}^{(i)}$
    \STATE Compute SVD: $\mat{U}^{(i)}, \mat{D}^{(i)}, \mat{V}^{(i)} \gets \text{SVD}(\mat{X}^{(i)}, \text{rank}=r_{\text{init}}^{(i)})$
    \STATE Calculate SV threshold $\tau^{(i)}$ (e.g., midpoint of $r$ and $r+1$ SVs)
\ENDFOR

\STATE \textit{// Phase 2: Score Space Segmentation}
\STATE Concatenate signal scores: $\mat{M} \gets [\mat{U}^{(1)}, \mat{U}^{(2)}, \dots, \mat{U}^{(k)}]$
\STATE Compute Joint SVD: $\mat{U}_{\text{joint}}, \mat{D}_{\text{joint}}, \mat{V}_{\text{joint}} \gets \text{SVD}(\mat{M})$

\IF{Joint Rank $r_{\text{joint}}$ is not provided}
    \STATE Compute Wedin bound $\delta_{\text{wedin}}$ and Random bound $\delta_{\text{rand}}$ via resampling
    \STATE Set cutoff $\delta_{\text{cutoff}} \gets \max(\delta_{\text{wedin}}, \delta_{\text{rand}})$
    \STATE $r_{\text{joint}} \gets \text{count}(\mat{D}_{\text{joint}}^2 > \delta_{\text{cutoff}})$
\ENDIF

\STATE Extract basis: $\mat{U}_{\text{joint}} \gets \mat{U}_{\text{joint}}[:, 1:r_{\text{joint}}]$
\STATE Remove columns from $\mat{U}_{\text{joint}}$ if projection norm on $\mat{X}^{(i)} < \tau^{(i)}$ (Identifiability Check)

\STATE \textit{// Phase 3: Final Decomposition}
\FOR{$i = 1$ to $k$}
    \STATE \textbf{Joint Estimate:} $\mat{J}^{(i)} \gets \mat{U}_{\text{joint}}\mat{U}_{\text{joint}}^T \mat{X}^{(i)}$
    
    \STATE \textbf{Individual Estimate:}
    \STATE Compute residual: $\mat{X}^{(i)}_{\text{orth}} \gets \mat{X}^{(i)} - \mat{J}^{(i)}$
    \STATE Compute SVD: $\mat{U}_{\text{ind}}, \mat{D}_{\text{ind}}, \mat{V}_{\text{ind}} \gets \text{SVD}(\mat{X}^{(i)}_{\text{orth}})$
    \STATE Determine rank $r_{\text{ind}}^{(i)}$ via threshold $\tau^{(i)}$
    \STATE $\mat{I}^{(i)} \gets \mat{U}_{\text{ind}}[:, :r_{\text{ind}}^{(i)}] \cdot \mat{D}_{\text{ind}}[:r_{\text{ind}}^{(i)}] \cdot \mat{V}_{\text{ind}}[:, :r_{\text{ind}}^{(i)}]^T$
    
    \STATE \textbf{Noise Estimate:} $\mat{E}^{(i)} \gets \mat{X}^{(i)} - \mat{J}^{(i)} - \mat{I}^{(i)}$
\ENDFOR

\STATE \textbf{Return} $\{\mat{J}^{(i)}, \mat{I}^{(i)}, \mat{E}^{(i)}\}_{i=1}^k$
\end{algorithmic}
\end{algorithm}

\newpage
\section{Validation of Spectral Synchronization via AJIVE}
\label{appendix:ajive}
\begin{figure}[!h]
    \centering
    \includegraphics[width=0.5\linewidth]{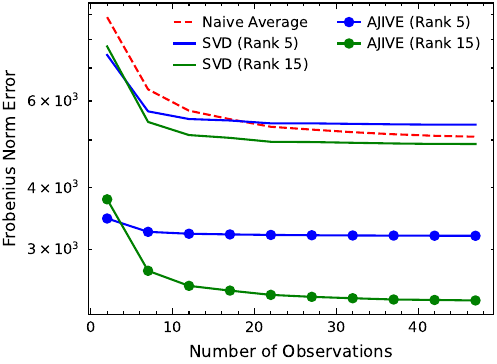} 
    \caption{\textbf{AJIVE recovers global geometry under structured drift.} We compare the reconstruction error of the global second moment $V^\star$ ($y$-axis) vs. number of clients ($x$-axis). \textbf{Naive Averaging (Red)} fails to converge to the true signal because the squared drift terms $\mathbb{E}[\mat{L}_k^2]$ introduce a persistent bias. \textbf{SVD on Average (Solid lines)} helps slightly but operates on already-corrupted data. \textbf{AJIVE (Markers)} successfully isolates the shared signal structure. Notably, AJIVE Rank-15 (Green circles) achieves the lowest error, correctly capturing the rank expansion induced by the element-wise square operation.}
    \label{fig:ajive_validation}
\end{figure}

To justify our design choice of using AJIVE for synchronizing the second-moment estimator $v$, we conduct a controlled synthetic experiment. The core challenge in federated adaptive optimization is that the second moment is a non-linear transformation of the gradients ($v \approx g^2$). Even if client gradients share a low-rank subspace, the client drift term $c_i$ introduces structured noise that does not vanish under simple averaging due to the quadratic nature of the update.

\subsection{Experimental Setup}
We simulate a scenario where the true global gradient signal lies in a low-rank subspace, but clients observe a drifted, noisy version. We aim to recover the \enquote{true} global second moment $\mat{V}^\star$.

\textbf{Data Generation.}
Let the global gradient signal be $\mat{G}^\star \in \mathbb{R}^{N \times M}$ with rank $r=5$. The ground truth second moment is $\mat{V}^\star = (\mat{G}^\star)^{\odot 2}$ (element-wise square). Note that while $
\mat{G}^\star$ is rank-5, its element-wise square $\mat{V}^\star$ typically has a higher rank (approximated by $r(r+1)/2 \approx 15$).
We simulate $K$ clients. Client $k$ observes a perturbed gradient:
\begin{equation}
    \mat{G}_k = \mat{G}^\star + \underbrace{\mat{L}_k}_{\text{Drift}} + \underbrace{\mat{\Xi}_k}_{\text{Noise}},
\end{equation}
where $\mat{L}_k$ is a random rank-2 perturbation representing client-specific drift, and $\mat{\Xi}_k \sim \mathcal{N}(0, \sigma^2)$ is unstructured noise. The client computes its local second moment $\mat{V}_k = (\mat{G}_k)^{\odot 2}$.

\textbf{Baselines.}
We compare three aggregation strategies to estimate $\mat{V}^\star$ given $\{V_k\}_{k=1}^K$:
\begin{enumerate}
    \item \textbf{Naive Averaging:} $\bar{\mat{V}} = \frac{1}{K} \sum_k \mat{V}_k$. This is the standard FedOpt approach.
    \item \textbf{Averaging + SVD:} We compute the naive average $\bar{\mat{V}}$ and then project it to rank $r'$ via SVD. This attempts to denoise the result \emph{after} aggregation.
    \item \textbf{AJIVE (Ours):} We apply AJIVE to the collection $\{\mat{V}_k\}_{k=1}^K$ to extract the joint spectral component shared across clients, effectively filtering out the drift $\mat{L}_k$ before it corrupts the signal.
\end{enumerate}

\subsection{Results and Analysis}
We measure the reconstruction error $\| \mat{V}_{\text{est}} - \mat{V}^\star \|_F$ as a function of the number of participating clients. The results are visualized in Figure~\ref{fig:ajive_validation}.

Our observations include: 

\begin{itemize}
    \item \textbf{Failure of Naive Averaging}: As shown by the red dashed line, simply averaging local states results in high error. This confirms our theoretical concern: non-linear drift does not cancel out linearly.
    \item \textbf{Rank Expansion}: The element-wise square of a rank-$r$ matrix has an effective rank much larger than $r$. Consequently, estimators restricted to the original rank (Rank 5, blue lines) underfit the geometry. AJIVE Rank 15 (green circles) matches the theoretical rank of the squared signal, achieving near-perfect recovery as the number of clients increases.
    \item \textbf{Robustness}: AJIVE outperforms post-hoc SVD (solid green line) because it uses the \emph{consensus} of subspaces across clients to identify the signal, whereas post-hoc SVD cannot distinguish between the signal and the dominant directions of the drift variance.
\end{itemize}

\newpage

\section{Details of Hyper Parameters}\label{appendix:hyperparameters}

We implement all methods in PyTorch and use Ray to simulate distributed clients.
For GLUE and DomainNet, we use $E{=}2$ local epochs per round and $R{=}50$ communication rounds; for Llama-2-7B we use $E{=}1$ and $R{=}20$ due to higher per-client cost. Additional details follow the pipeline configs in \texttt{conf/pipeline/*.yaml}.

\begin{table}[!h]
\centering
\caption{\textbf{Federated protocol and core training hyperparameters.} GLUE/DomainNet use $E{=}2$ and $R{=}50$; Llama uses $E{=}1$ and $R{=}20$. All runs use bf16.}
\label{tab:hparams_train}
\footnotesize
\setlength{\tabcolsep}{3pt}
\renewcommand{\arraystretch}{1.10}
\resizebox{\columnwidth}{!}{%
\begin{tabular}{lcccccccc}
\toprule
\textbf{Domain} & \textbf{Model} & \textbf{LR} & \textbf{BS (tr/ev)} & \textbf{GA} & \textbf{Warmup} & \textbf{WD} & \textbf{Sched} & \textbf{Len} \\
\midrule
GLUE & RoBERTa-base \citep{liu2019roberta} & $5\!\times\!10^{-5}$ & 32/32 & -- & 0.06 & 0.01 & cosine & -- \\
DomainNet & ViT-base \citep{dosovitskiy2020vit} & $2\!\times\!10^{-4}$ & 128/128 & -- & 0.05 & 0.1 & cosine & -- \\
MetaMathQA & Llama-2-7B \citep{touvron2023llama} & $1\!\times\!10^{-4}$ & 2/2 & 32 & 0.03 & 0.01 & cosine & 1024 \\
\bottomrule
\end{tabular}%
}
\vspace{-2pt}
\end{table}

\begin{table}[!h]
\centering
\caption{\textbf{LoRA configuration and target modules.}}
\label{tab:hparams_lora}
\footnotesize
\setlength{\tabcolsep}{3pt}
\renewcommand{\arraystretch}{1.10}
\resizebox{\columnwidth}{!}{%
\begin{tabular}{lcccll}
\toprule
\textbf{Domain} & $r$ & $\alpha$ & \textbf{Task} & \textbf{Modules to save} & \textbf{Target modules} \\
\midrule
GLUE & 8 & 16 & SEQ\_CLS & \texttt{["classifier"]} & \texttt{["query","value","dense"]} \\
DomainNet & 16 & 32 & SEQ\_CLS & \texttt{["classifier"]} & \texttt{["query","value","dense"]} \\
MetaMathQA & 32 & 64 & CAUSAL\_LM & \texttt{null} & \texttt{["q\_proj","k\_proj","v\_proj","o\_proj","gate\_proj","up\_proj","down\_proj"]} \\
\bottomrule
\end{tabular}%
}
\end{table}

\begin{table}[!h]
\centering
\caption{\textbf{Compute allocation.} GPUs per run in our Ray-based client simulation.}
\label{tab:hparams_resource}
\footnotesize
\setlength{\tabcolsep}{6pt}
\renewcommand{\arraystretch}{1.10}
\begin{tabular}{lc}
\toprule
\textbf{Domain} & \textbf{\#GPUs} \\
\midrule
GLUE & 0.5 \\
DomainNet & 0.5 \\
MetaMathQA (Llama-2-7B) & 1 \\
\bottomrule
\end{tabular}
\end{table}

\section{Dirichlet Distribution for Modeling Non-IID Data}
\label{appendix:lda_non_iid}

\begin{figure}[!h]
    \centering
    \begin{subfigure}[b]{0.32\textwidth}
        \centering
        \includegraphics[width=\textwidth]{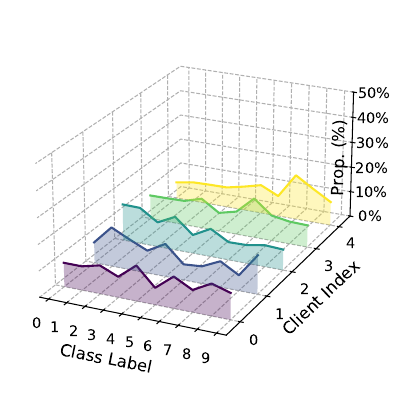} 
        \caption{$\alpha = 10$}
        \label{fig:alpha_0.1}
    \end{subfigure}
    \hfill
    \begin{subfigure}[b]{0.32\textwidth}
        \centering
        \includegraphics[width=\textwidth]{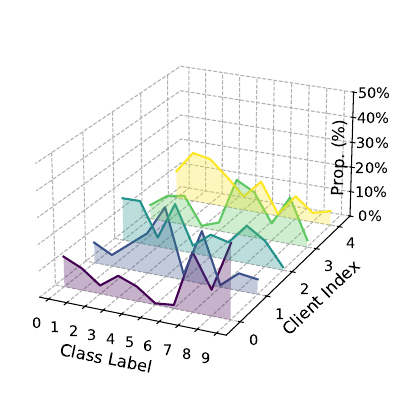}
        \caption{$\alpha = 1$}
        \label{fig:alpha_0.5}
    \end{subfigure}
    \hfill
    \begin{subfigure}[b]{0.32\textwidth}
        \centering
        \includegraphics[width=\textwidth]{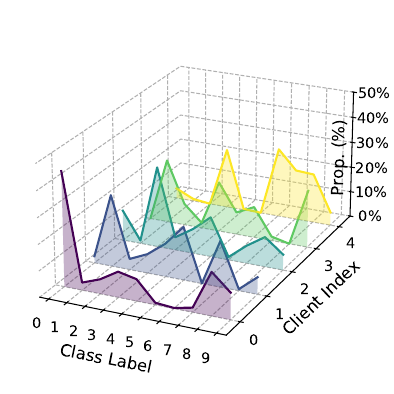}
        \caption{$\alpha = 0.5$}
        \label{fig:alpha_10}
    \end{subfigure}
    \caption{\textbf{Impact of concentration parameter $\alpha$ on data distribution.} We illustrate the class distribution across clients for varying $\alpha$ on MNIST datasets. Smaller values (e.g., $\alpha=0.5$) result in severe class imbalance, where clients possess samples from very few classes.}
    \label{fig:placeholder}
\end{figure}

Following common practice in federated learning \citep{yurochkin2019lda, wang2020federated_lda, li2020practical_lda}, we model client data heterogeneity using a Dirichlet distribution. Specifically, for a classification task with $K$ classes, the label distribution for each client $i$ is sampled from a Dirichlet distribution:  
\[
p_i \sim \text{Dir}(\alpha \cdot \mathbf{1}_K),
\]  
where $\alpha > 0$ is the concentration parameter and $\mathbf{1}_K$ is a $K$-dimensional vector of ones. A smaller $\alpha$ produces more skewed client label distributions (i.e., stronger non-IID conditions), while a larger $\alpha$ yields more balanced distributions approaching the IID case.  

After sampling $p_i$, we partition the dataset by allocating examples to client $i$ according to $p_i$. Figure~\ref{fig:placeholder} illustrates how the Dirichlet partition leads to distinct label distributions across clients. For MetaMathQA\citep{yu2023metamath}, we treat \enquote{type} as the label. 

In our experiments, we set $\alpha=0.5$ to study the impact of severe data heterogeneity.

\section{More Results of GLUE}\label{app:results_of_glue}

\begin{table*}[!h]
\caption{GLUE Benchmark Results. \ding{51} denotes IID. \ding{53} denotes Non-IID, $\Delta$ indicates the difference.}
\label{tab:glue2}
\begin{tabularx}{\textwidth}{@{}l *{4}{XXX}@{}}
\toprule
\textbf{Method} & 
\multicolumn{3}{c}{\textbf{MNLI} Acc\%} & 
\multicolumn{3}{c}{\textbf{QNLI} Acc\%} & 
\multicolumn{3}{c}{\textbf{RTE} Acc\%} & 
\multicolumn{3}{c}{\textbf{STS-B} MSE} \\
\cmidrule(lr){2-4} \cmidrule(lr){5-7} \cmidrule(lr){8-10} \cmidrule(lr){11-13}
 & {\ding{51}} & {\ding{53}} & {$\Delta$} 
 & {\ding{51}} & {\ding{53}} & {$\Delta$} 
 & {\ding{51}} & {\ding{53}} & {$\Delta$} 
 & {\ding{51}} & {\ding{53}} & {$\Delta$} \\
\midrule
FedAvg-Full  & 82.5 & 81.1 & $\downarrow$1.4 & 90.1 & 89.8 & $\downarrow$0.3 & 75.1 & 68.2 & $\downarrow$6.9 & 0.45 & 0.67 & $\downarrow$.22 \\

\midrule

FedIT  & 75.6 & 69.8 & $\downarrow$5.8 & 83.4 & 78.1 & $\downarrow$5.3 & 66.5 & 58.9 & $\downarrow$7.6 & 0.38 & 0.55 & $\uparrow$.17 \\
FFA-LoRA  & 82.2 & 75.8 & $\downarrow$6.4 & 89.9 & 83.4 & $\downarrow$6.5 & 74.8 & 65.3 & $\downarrow$9.5 & 0.44 & 0.56 & $\uparrow$.16\\
FLoRA  & \textbf{82.7} & 73.2 & $\downarrow$9.5 & 92.0 & 86.7 & $\downarrow$5.3 & 75.1 & \textbf{68.2} & $\downarrow$6.9 & 0.39 & 0.53 & $\uparrow$.14 \\
LoRA-Fair  & 81.3 & 80.0 & $\downarrow$1.3 & 83.1 & 81.8 & $\downarrow$1.3 & 71.3 & 63.2 & $\downarrow$8.1 & 0.57 &  0.61 &  \textbf{$\uparrow$.04} \\
FR-LoRA  & 80.9  & 72.3 & $\downarrow$8.6 & 93.1 & 90.8 & $\downarrow2.3$ & \textbf{76.2} & 70.4 & $\downarrow$5.8 & \textbf{0.33} & \textbf{0.39} & $\uparrow$.06 \\

\midrule
\methodname{}$^-$ & 81.2 & 77.4 & $\downarrow$3.8 & \textbf{93.3} & 88.1 & $\downarrow$5.2 & \textbf{76.2} & 68.0 & $\downarrow$8.2 & \textbf{0.33} & \textbf{0.51} & $\uparrow$.18 \\
\methodname{}  & \textbf{82.7} & \textbf{81.4} & \textbf{$\downarrow$1.3} & 90.3 & \textbf{89.8} & \textbf{$\downarrow$0.5} & 75.2 & \textbf{68.2} & \textbf{$\downarrow$7.0} & 0.45 & 0.67 & $\uparrow$.22 \\
\bottomrule
\end{tabularx}
\end{table*}


\end{document}